\documentclass[runningheads]{llncs}

\usepackage[utf8]{inputenc}
\usepackage{microtype}
\usepackage{graphicx}
\usepackage{hyperref}
\usepackage{changepage}
\usepackage{hhline}
\usepackage{hyperref}

\usepackage{cite}
\usepackage{comment}
\usepackage{amsmath}
\usepackage{amssymb}
\usepackage{hyperref}
 \usepackage{xcolor}
\usepackage{hhline}
\usepackage{subfig}

\usepackage[font=small]{caption}
\usepackage{algorithm}
\usepackage{algpseudocode}

\newcommand{\reals}{\ensuremath{\mathbb{R}}}
\newcommand{\naturals}{\ensuremath{\mathbb{N}}}

\newcommand{\argmin}{\mathop{\arg\,\min}}
\newcommand{\dom}{\text{dom}}

\newcommand{\E}{\mathop{\mathbb{E}}}

\newcommand{\vc}[1]{\ensuremath{\boldsymbol{#1}}}

\newcommand{\outl}{P_{\texttt{out}}}
\newcommand{\dc}{\alpha}
\newcommand{\mboa}{\texttt{MBOSADM}}
\newcommand{\mbob}{\texttt{MBOSGD}}
\newcommand{\sgd}{\texttt{SGD}}
\newcommand{\enc}{\mathtt{enc}}
\newcommand{\dec}{\mathtt{dec}}

\DeclareMathOperator{\sign}{sign}

\begin{document}
\title{Robust Regression via Model Based Methods\thanks{The authors gratefully acknowledge support from the National Science Foundation (Grants  CCF-1750539, IIS-1741197, and  CNS-1717213), DARPA (Grant HR0011-17-C-0050), and a research grant from American Tower Corp.}}
\author{Armin Moharrer\orcidID{0000-0002-8374-7286} \and
Khashayar Kamran\orcidID{0000-0002-4086-1038}\and
Edmund Yeh\orcidID{0000-0002-9544-1567}
\and
Stratis Ioannidis\orcidID{0000-0001-8355-4751} 
}

\authorrunning{A. Moharrer et al.}
\institute{Northeastern University, Boston MA 02115, USA \\
\email{\{amoharrer,kamrank,eyeh,ioannidis\}@ece.neu.edu}}
%

%Uncomment in short version (ECML-PKDD)
\newcommand{\fullversion}[2]{#1}
%Uncomment in full version (arXiv)
%\newcommand{\fullversion}[2]{#2}

\maketitle              % typeset the header of the contribution
\begin{abstract}
%We introduce  a class of non-smooth and non-convex  loss functions that are robust to outliers. In particular,  we 
%replace the widely-used mean squared loss with $\ell_p$ norms; despite computation advantages due to differentiability, the mean squared loss is not robust to outliers, as the outliers raised to the power of two  dominate the objective. 
The   mean squared error loss is widely used in many applications, including auto-encoders, multi-target regression, and matrix factorization, to name a few. Despite computational advantages due to its differentiability, it is not robust to outliers. In contrast, $\ell_p$ norms are known to be robust, but cannot be optimized via, e.g., stochastic gradient descent, as they are non-differentiable. We propose an  algorithm inspired by so-called model-based optimization  (MBO) \cite{Ochs2019,pmlr}, which replaces a non-convex objective with a convex model function and alternates between optimizing the model function and updating the solution. 
We apply this to robust regression,  proposing SADM, a stochastic variant of  the Online Alternating Direction Method of Multipliers (OADM) \cite{wang2013online} to solve the inner optimization in MBO. We show that  SADM converges with the rate $O(\log T/T)$. 
Finally, we demonstrate experimentally  (a) the robustness of $\ell_p$ norms to outliers and (b) the efficiency of our proposed model-based algorithms in  comparison with gradient  methods on autoencoders and multi-target regression. 

%\keywords{Robust Regression \and Stochastic Optimization\and Non-smooth Non-convex Optimization\and ADMM.}
\end{abstract}
\section{Introduction}\label{sec:intro}
 Mean Squared Error (MSE) loss problems are ubiquitous  in machine learning and data mining. Such problems have the following form:
 \begin{equation}\label{eq:MSSE}
\min_{{\theta}}\frac{1}{n} \sum_{i=1}^n \|F(\vc{\theta}; \vc{x}_i)\|_2^2+ g(\vc{\theta}),
\end{equation}
where  function $F:\reals^d\times \reals^m\to \reals^N$ captures the contribution of a  sample  $\vc{x}_i\in\reals^m, i=1,\ldots, n$, to the objective under the parameter $\vc{\theta}\in \reals^d$ and $g:\reals^d\to \reals$ is a regularizer. Example applications include training auto-encoders \cite{jiang2016l2,mehta2016stacked}, matrix factorization \cite{gillis2014nonnegative}, and multi-target regression \cite{waegeman2019multi}.

\begin{figure}
    \subfloat[Avg. Non-outliers  Loss]{\includegraphics[width=0.33\textwidth]{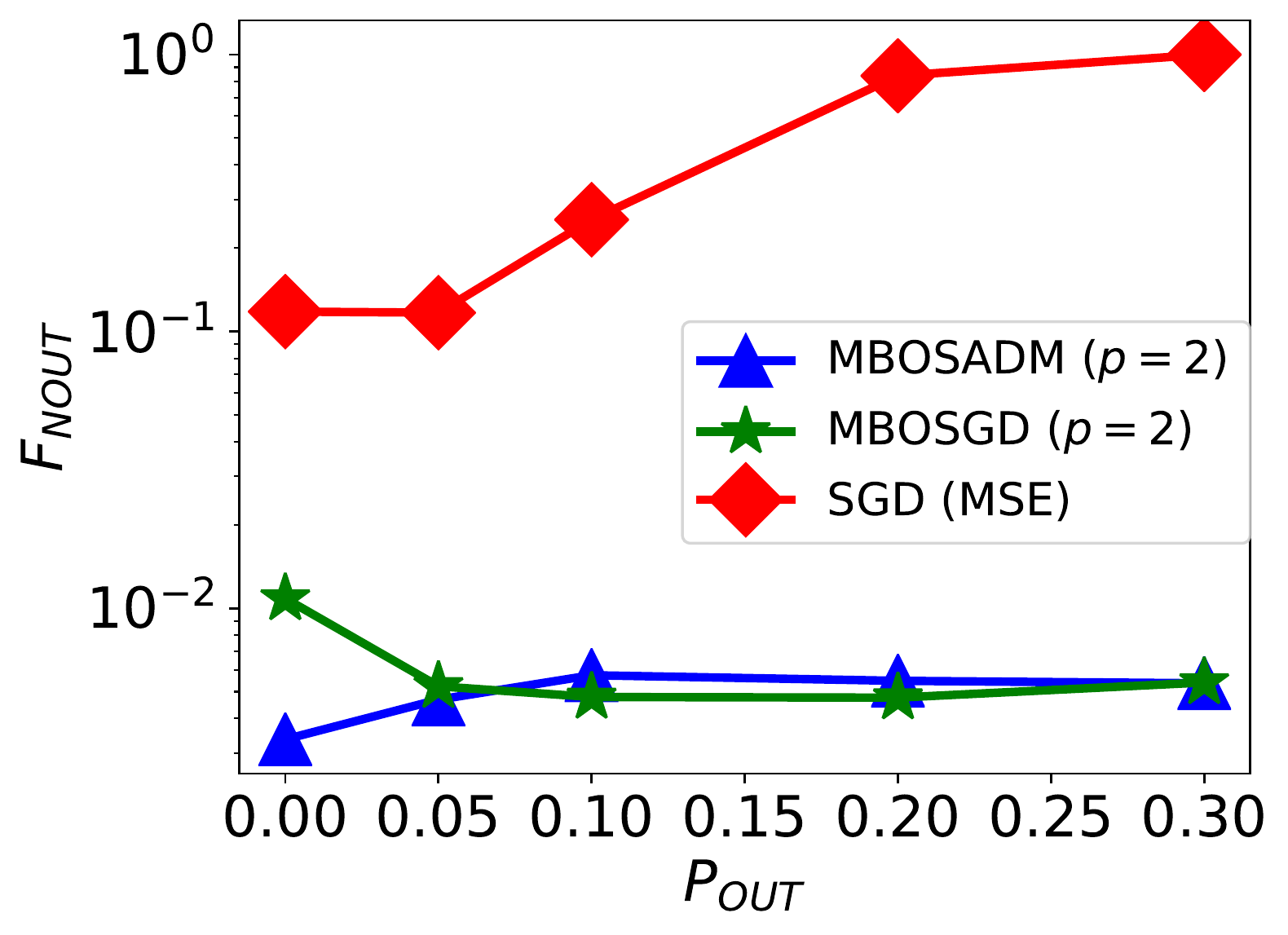}\label{fig:intro_nout}}
    \subfloat[Avg. Test Loss]{\includegraphics[width=0.33\textwidth]{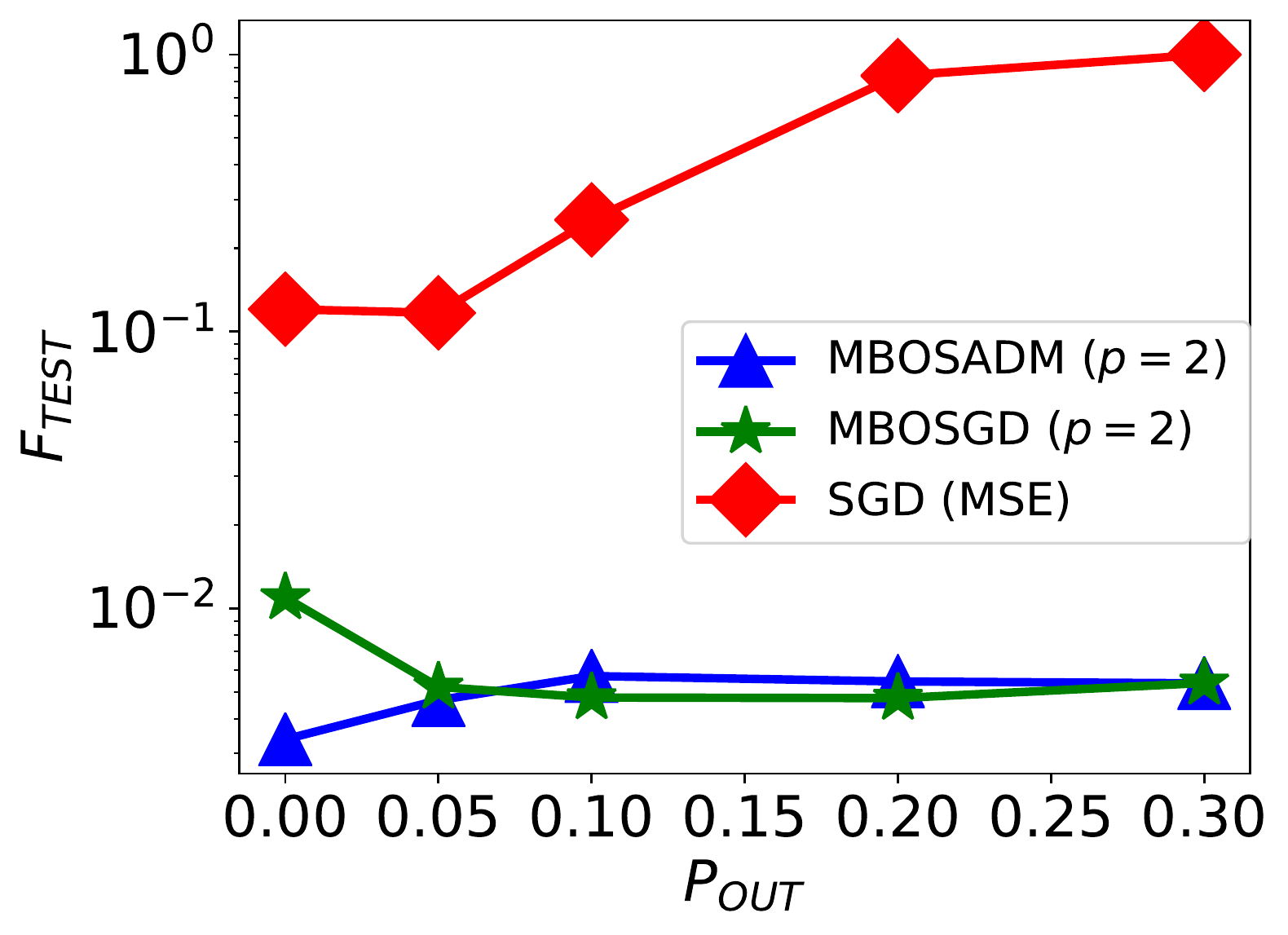}\label{fig:intro_test}}
    \subfloat[Accuracy]{\includegraphics[width=0.33\textwidth]{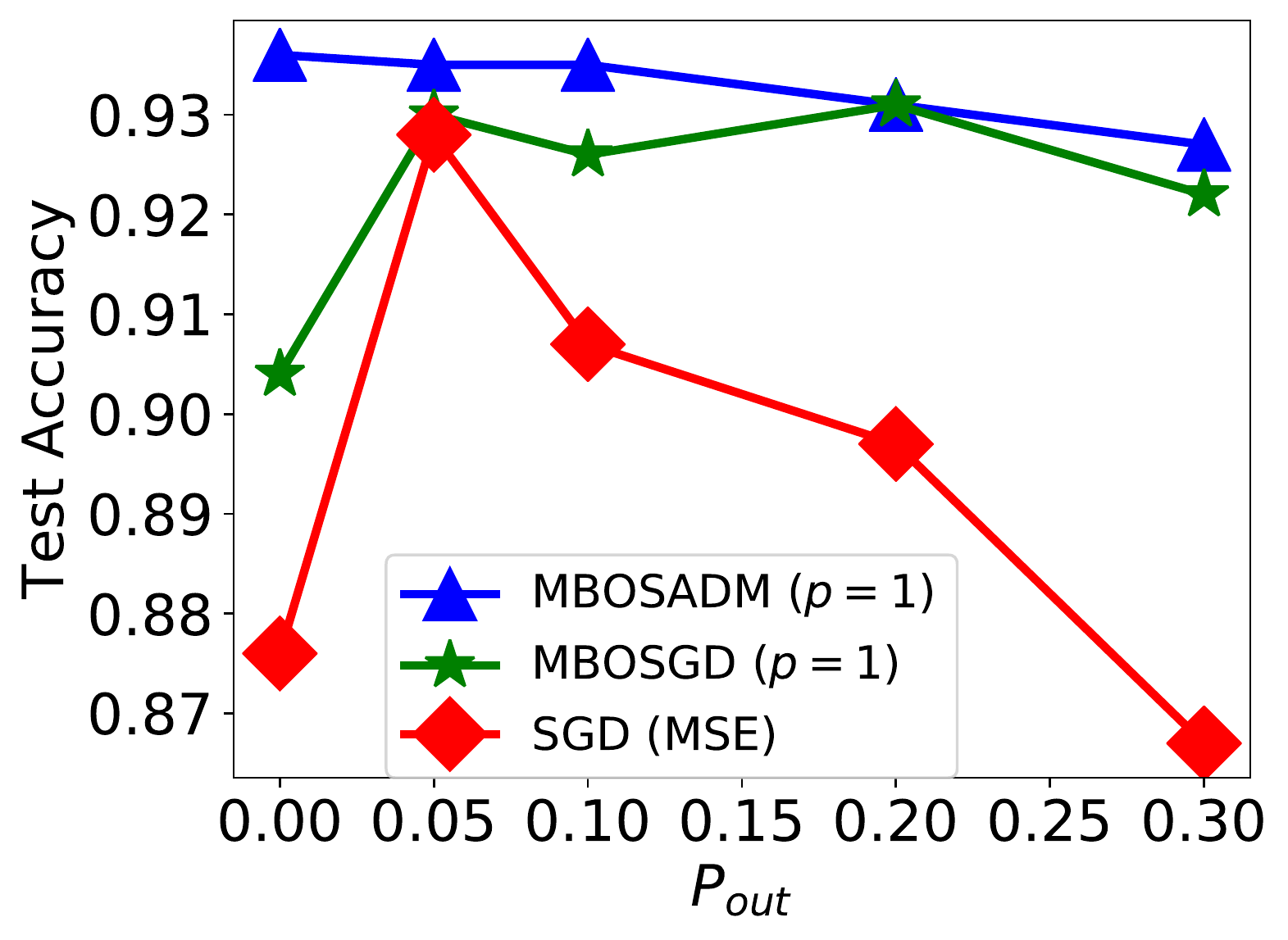}\label{fig:intro_acc}}
    \caption{Robustness of $\ell_p$ norms vs.~MSE to outliers introduced to \texttt{MNIST}  when training an autoencoder. Figures~\ref{fig:intro_nout} and \ref{fig:intro_test} show the average loss over the  non-outliers and the test set, respectively;  values in each figure are normalized w.r.t. the largest value. The test accuracy of a logistic regression on the latent features is shown in Fig.~\ref{fig:intro_acc}.  We see that, under MSE, both for the loss values and classification accuracy are significantly affected by the fraction of outliers  $\outl.$ Robust embeddings under $p=1,2$ norms optimized via our proposed MBO methods exhibit almost constant behavior w.r.t. $\outl.$}\label{fig:demo}
\end{figure}

The MSE loss in \eqref{eq:MSSE} is computationally convenient, as the resulting problem is smooth and can thus be optimized efficiently via gradient methods, such as  stochastic gradient descent (SGD). However, it is well-known that the MSE loss is not robust to \emph{outliers} \cite{friedman2001elements,ding2006r,nie2010efficient,kong2011robust,ke2005robust}, i.e., samples   far from the dataset mean. Intuitively, when squaring the error,  outliers  tend to  dominate the objective. 
 To mitigate the effect of outliers, a classic  approach is to  introduce robustness by replacing the squared error with either the $\ell_2$ norm \cite{ding2006r,nie2010efficient,qian2013robust,jiang2016l2,du2015robust,mehta2016stacked} or the $\ell_1$ norm \cite{eriksson2010efficient,croux1998robust,ke2005robust,kwak2008principal,li2010l1,baccini1996l1,pesme2020l1loss,kong2011robust}. This has been  applied to several applications, including feature selection \cite{nie2010efficient,qian2013robust},  PCA \cite{kwak2008principal,li2010l1,baccini1996l1,ding2006r}, K-means clustering \cite{du2015robust}, training autoencoders \cite{jiang2016l2,mehta2016stacked}, matrix factorization \cite{eriksson2010efficient,croux1998robust,ke2005robust,kong2011robust}, and regression \cite{pesme2020l1loss}. 
% Furthermore, for a class of  non-Gaussian noise, Kuruo{\u{g}}lu et al. \cite{kuruouglu1998least} show that the $\ell_2$ square norm fails and propose a least $\ell_p$ square formulation and showed its superior performance in the context of signal processing.
 Motivated by this approach, we study  the following robust variant of Problem~\eqref{eq:MSSE}:
  \begin{align}\label{eq:intro_problem}
 \min_{{\theta}}\frac{1}{n} \sum_{i=1}^n \|F(\vc{\theta}; \vc{x}_i)\|_p+ g(\vc{\theta}),
\end{align}
where $\|\cdot\|_p$ denotes an  $\ell_p$ norm ($p\geq1$). 
We are particularly interested in cases where $F$ is not affine and, in general, Problem~\eqref{eq:intro_problem} is non-convex.   
This includes, e.g., feature selection \cite{nie2010efficient}, matrix factorization \cite{kong2011robust,eriksson2010efficient}, auto-encoders \cite{jiang2016l2}, and deep multi-target regression \cite{pesme2020l1loss,waegeman2019multi}.

A significant challenge behind solving Prob.~\eqref{eq:intro_problem} is that its objective is not smooth, precisely because the $\ell_p$ norm is not differentiable at $\vc{0}\in \reals^N.$ For non-convex and non-smooth problems of the form \eqref{eq:intro_problem}, where the objective contains a composite function, \emph{Model-Based Optimization} (MBO) methods \cite{lewis2016proximal, davis2019proximally,drusvyatskiy2018error,davis2019stochastic,drusvyatskiy2019efficiency,Ochs2019} come with good experimental performance as well as theoretical guarantees. In particular, these MBO methods 
define a convex (but non-smooth) approximation of the  main objective, called the \emph{model function}. They then iteratively optimize this model function plus a proximal quadratic term. Under certain conditions, MBO  converges to a stationary point of the non-convex problem \cite{lewis2016proximal}. 

In this work, we use MBO to solve Problem~\eqref{eq:intro_problem} for arbitrary $\ell_p$ norms. In particular, each MBO iteration results in a convex optimization problem. We solve these sub-problems using a novel 
 stochastic variant of the Online Alternating Direction Method  (OADM) \cite{online_admm}, which we call \emph{Stochastic Alternating Direction Method} (SADM). Using SADM is appealing, as  its resulting steps  have efficient gradient-free solutions; in particular, we exploit a bisection method \cite{liu2010efficient, moharrer2020massively} for finding the proximal operator of $\ell_p$ norms. 
 We provide theoretical guarantees for  SADM. 
As an additional benefit, SADM comes with a stopping criterion, which is hard to obtain for gradient methods when the objective is  non-smooth  \cite{mai2020convergence}.

Overall, we make the following contributions: 
\begin{itemize}
    \item We study a general outlier-robust optimization that replaces the MSE with $\ell_p$ norms. We show that such problems can be solved via Model-Based Optimization (MBO) methods.  
    \item We propose SADM, i.e., a stochastic version of OADM, and show that under strong convexity of the regularizer $g$, it converges with a $O(\log T/T)$ rate  when solving the sub-problems arising at each MBO iteration.
    \item We conduct extensive experiments on training auto-encoders and multi-target regression. We show  (a) the higher robustness of $\ell_p$ norms in comparison with  MSE and (b) the superior performance of MBO, against stochastic gradient methods, both in terms of minimizing the objective and performing down-stream classification tasks. In some cases, we see that the MBO variant using SADM obtains objectives that are 29.6$\times$ smaller than the ones achieved by the competitors. 

\end{itemize}
The performance of our MBO approach is illustrated in Fig.~\ref{fig:demo}. An autoencoder trained via SGD over the MSE objective is significantly affected by the presence of outliers; in contrast, our MBO methods applied to $\ell_p$ objectives are robust to outliers. These relative benefits are also evident  in a downstream classification task over the latent embeddings.
The remainder of this paper is organized as follows. We review related work in Sec.~\ref{sec:related}. We introduce our robust formulation along with its applications in Sec.~\ref{sec:probapp}. We describe  the instance of  MBO applied to our problem in Sec.~\ref{sct:MBO}. We  introduce SADM and its convergence analysis in Sec.~\ref{sec:oadm} and present our experiments in Sec.~\ref{sec:experiments}. We finally conclude in Sec.~\ref{sec:conclusion}.

\section{Related Work}\label{sec:related}

\noindent \textbf{Robustness of $\ell_p$ Norms}: To improve the sensitivity of MSE to outliers,  Ding et al. \cite{ding2006r} first suggested replacing the MSE with the $\ell_2$ norm in the context of Principal Component Analysis (PCA). This motivated a line of research for developing robust algorithms using the $\ell_2$ norm  in different applications, e.g., %It has successfully been applied to a variety of applications, such as,  
 non-negative matrix factorization \cite{kong2011robust},  feature selection \cite{qian2013robust,nie2010efficient}, training autoencoders \cite{jiang2016l2}, and $k$-means clustering \cite{du2015robust}. Attaining robustness via the $\ell_1$ norm has also been used in matrix factorization \cite{eriksson2010efficient,croux1998robust,ke2005robust}, PCA \cite{kwak2008principal,li2010l1,baccini1996l1}, and regression \cite{pesme2020l1loss}. Robustness of  the $\ell_1$ norm can be linked to robustness of median to outliers in comparison to average value (see, e.g., %Chapter 2 of the book by
 Friedman et al.~\cite{friedman2001elements}). % Furthermore, Kuruo{\u{g}}lu et al. \cite{kuruouglu1998least} provide another incentive for using $\ell_p$ norms. They show that for a class of non-Gaussian noise which agrees tremendously with empirical data, a least $\ell_p$ square formulation performs better than $\ell_2$ square norm for polynomial filter estimation. 
Our problem includes robust variations considered in, e.g., \cite{pesme2020l1loss,nie2010efficient,jiang2016l2,kong2011robust,eriksson2010efficient},  as special cases. However, these earlier algorithms  are  tailored to specific $\ell_p$ norms  and/or  do not  generalize beyond the  studied  objective or application (some  works, e.g., \cite{nie2010efficient,pesme2020l1loss}, only consider convex problems). In contrast, we unify  these variations for different applications  as a non-convex and non-smooth problem, and present  a general optimization algorithm for arbitrary  $\ell_p$ norms. 
 
 %Motivated by these works, here we replace the $\ell_2$ squared norm with general $\ell_p$ ($p\geq 1$) norms.
 %In our experiments, we showcase the robustness of $\ell_p$ norm to outliers for different values of $p$. As we discuss throughout the paper, our method can be used for problems with any $\ell_p$ norms and $p \geq 1$.
 \noindent \textbf{Non-smooth/non-convex Optimization}: Non-smooth and non-convex optimization problems arise in many applications, such as non-negative matrix factorization \cite{gillis2014nonnegative}, compressed
sensing with non-convex norms \cite{attouch2010proximal}, and $\ell_p$ norm regularized sparse regression problems   \cite{natarajan1995sparse,blumensath2009iterative}.
A class of non-smooth non-convex optimization problems, known as \emph{weakly convex problems} \cite{vial1983strong}, i.e.,  problems in which the objective function is the  sum of a convex function and a quadratic function, have attracted a lot of attention \cite{lewis2016proximal,duchi2018stochastic,drusvyatskiy2019efficiency,davis2019stochastic,le2020inertial, mai2020convergence}. Mai and Johansson \cite{mai2020convergence} provided novel theoretical guarantees on the convergence of stochastic gradient descent with momentum for weakly-convex functions. However, in our experiments in Sec.~\ref{sec:experiments}, we show that model-based methods considerably outperform these stochastic gradient methods with momentum. %
%In particular, the objective of our problem \eqref{eq:problem}, which is a composition of a non-smooth function with a smooth function is weakly convex \cite{drusvyatskiy2019efficiency}. 

Our approach  falls under the class of  \emph{prox-linear} methods \cite{lewis2016proximal,duchi2018stochastic,drusvyatskiy2018error,drusvyatskiy2019efficiency,davis2019stochastic,le2020inertial}, that  solve problems where the objective is a composition of a non-smooth  convex function and a smooth function, exactly as in  Prob.~\eqref{eq:intro_problem}.  %Functions with this composite form are weakly convex \cite{drusvyatskiy2019efficiency}. 
Such methods iteratively minimize the composition of the non-smooth function with the first-order approximation of the smooth function \cite{lewis2016proximal,duchi2018stochastic,drusvyatskiy2019efficiency,davis2019stochastic}. 
Lewis and Wright \cite{lewis2016proximal}   prove convergence to a stationary point while Drusvyatskiy et al.~prove linear convergence \cite{drusvyatskiy2018error} and obtain sample complexity guarantees %when  sub-problems are solved via  gradient methods 
\cite{drusvyatskiy2019efficiency}. 
Ochs et al.~\cite{Ochs2019,pmlr} generalize prox-linear methods 
by proposing \emph{Model-Based Optimization} (MBO) for  both smooth and non-smooth non-convex problems. %They  define a model function, i.e., a convex approximation to the main non-convex objective; 
MBO reduces to a prox-linear method when the objective has a composite form, as in our case. Ochs et al.~further considered  non-quadratic proximal penalties in sub-problems and complemented MBO  with an  Armijo-like line search. We leverage both their  line search and  theoretical guarantees (c.f.~Prop.~\ref{prop:ochs}); our main technical departure is in solving sub-problems per iteration via  SADM, which we discuss next. %In particular, Ochs et al. obtain the prox-linear algorithms as a special  case of their model-based optimization method for problems with composite form. 
% Le et al. \cite{le2020inertial} also proposed a proximal block coordinate descent methods for solving  composite optimization problems, where their focus was on matrix factorization. 
%We solve our proposed problem via the MBO method proposed by Ochs et al.~\cite{Ochs2019}, which results in  non-smooth convex inner-problems similar to  prox-linear methods, due to the composite form of our objective. 

\noindent \textbf{ADMM.}   The Alternating Direction Method of Multipliers (ADMM) \cite{boyd2011distributed} is a convex optimization algorithm that provides efficient methods for non-smooth problems. Applying ADMM  often  results in  sub-problems that can be solved efficiently via   proximal operators \cite{boyd2011distributed,peng2012rasl,tao2011recovering}. To speed up ADMM, stochastic variants \cite{zheng2016fast,liu2017accelerated,ouyang2013stochastic} have been proposed for minimizing  sum-like objectives. These stochastic variants, similar to SGD, update solutions using  the gradients of a  small batch of terms in the objective, at each iteration. 
Another group of works proposed  online variants of ADMM \cite{online_admm,suzuki2013dual,hosseini2014online}. In these  variants, the goal is to minimize the summation of loss functions that are revealed  by an adversary. 

 Wang and Banerjee  \cite{online_admm} proposed the first online variant of ADMM, termed  Online Direction Method of Multipliers (OADM). %, operating in an adversarial setting. 
% Suzuki \cite{suzuki2013dual} then provided  other  online ADMM variants that use the gradients of the functions that the adversary reveals.  
 %OADM focus on the \emph{online} learning setting, where at each an \emph{avdersary} reveals  a loss term, and the goal is to minimize the total loss for the terms sampled by the avdersary. 
  Here, we  propose a  stochastic version of OADM, {Stochastic Alternating Direction Method} (SADM), to solve inner-problems in MBO iterations. SADM is  similar to OADM with the  difference that functions are sampled uniformly at random and are not given by an adversary. 
  We prove that SADM converges with a $O(\log T/T)$ rate  when the regularizer is strongly convex. Other existing stochastic or online ADMM variants either require a  smooth objective \cite{zheng2016fast,liu2017accelerated} or bounded sub-gradients \cite{ouyang2013stochastic,suzuki2013dual}, neither of which apply for the inner problems we solve. In contrast, we  show that applying SADM results in sub-problems that admit gradient-free efficient solutions  
   via a bisection method for finding proximal operators of $\ell_p$ norms \cite{liu2010efficient,moharrer2020massively}.

\section{Robust Regression and Applications}\label{sec:probapp}

\noindent\textbf{Notations.}
Lowercase boldface letters represent vectors, while capital boldface letters represent matrices. %For a matrix $\vc{D}\in \reals^{N\times d}$ we denote  its $i$-th row by $\vc{D}_i\in \reals^d$ and its $j$-th column by $\vc{D}^j\in \reals^N.$
%We denote the Jacobian of the function $F$ evaluated at a point $x$ with $\vc{D}_{F_i}(x)\in \reals^{N \times d}.$ 
We also use the notation $[n] \triangleq \{1,2,\dots,n \}$.
 
\noindent\textbf{Robust Regression.} 
We first  extend Prob.~\eqref{eq:intro_problem} to include constraints via:
\begin{equation} \label{eq:problem}
 \min_{\vc{\theta}}\frac{1}{n} \sum_{i\in[n]} \|F(\vc{\theta}; \vc{x}_i)\|_p+ g(\vc{\theta}) +\chi_{\mathcal{C}}(\vc{\theta}),
\end{equation}
where, again,  $F:\reals^d \times \reals^m \rightarrow \reals^N$ is smooth, % (but not necessarily convex), and
$|| \cdot ||_p$ is the $\ell_p$ norm,   $g:\reals^d\rightarrow \reals$ is a convex regularizer such that inf $g > -\infty$,  while $\chi_{\mathcal{C}}:\reals^d\rightarrow \{0, \infty\}$ is the indicator function of the convex set $\mathcal{C}\subseteq \reals ^d.$ In practice, we are often interested in cases where  either the regularizer or the constraint is absent.

\noindent\textbf{Applications.} For the sake of concreteness, we introduce some applications of Prob.~\eqref{eq:problem}.
Function $g$ is typically either the lasso  (i.e., the $\ell_1$ norm  $g(\vc{\theta})=\|\vc{\theta}\|_1$)  or ridge regularizer (i.e., the $\ell_2$ norm squared  $g(\vc{\theta})=\|\vc{\theta}\|_2^2$).  We thus focus on the definition of $F(\cdot;\cdot)$ and constraint set $\mathcal{C}$ in each of these applications.

%\begin{itemize}

%\item 
\noindent\emph{Auto-encoders \cite{jiang2016l2}.} Given $n$ data points $\vc{x}_i\in\reals^m$, $i\in [n]$, auto-encoders embed them in a $m^\prime-$dimensional space, $m'\ll m$, as follows. The mapping to $\reals^{m'}$ is done by a  possibly non-linear function (e.g., a neural network) with $d_{\enc}$ parameters $F_{\enc}:\reals^{d_{\enc}}\times \reals^m\rightarrow \reals^{m^\prime}$, called the \emph{encoder}. An inverse mapping, the \emph{decoder} $F_{\dec}:\reals^{d_{\dec}}\times \reals^{m^\prime} \rightarrow \reals^ m$  with $d_{\dec}$ parameters re-constructs the original points given latent embeddings. %$F_{\dec}$ is usually called the \emph{decoder}. 
%Note that the output of the decoder is not an exact reconstruction of the original input $\vc{x}_i,$ but rather a prior from which the original inputs may be generated with high probability.  For instance, each data point $\vc{x}_i$ is constructed as follows
 %\begin{align*}
 %    \vc{\hat{x}}_i = F_{\dec}(F_{\enc}(\vc{x}_i)) + \vc{\epsilon},
 %\end{align*}
%where $\vc{\epsilon}\in \reals^m$ is a random additive noise with zero-mean and $\sigma^2 \vc{I}$ covariance matrix ($\vc{I}\in \reals^{m \times m}$ is the identity matrix).  
 Both the encoder and the decoder are trained jointly over a dataset $\{\vc{x}_i\}_{i=1}^n$ by minimizing the reconstruction error; cast in our robust setting, this amounts to minimizing~\eqref{eq:problem} with \begin{align}F(\vc{\theta}; \vc{x}_i) = \vc{x}_i - F_{\dec}\left(\vc{\theta}_{\dec}; F_{\enc}(\vc{\theta}_{\enc}; \vc{x}_i) \right),\end{align}
 where $\vc{\theta}=[\vc{\theta}_{\dec};\vc{\theta}_{\enc}]\in \reals ^{d_{\enc} +d_{\dec}}$ comprises the parameters of the encoder and the decoder.
 Robustness here aims to ameliorate the effect of outliers in the dataset $\{\vc{x_i}\}_{i=1}^n$. 
 %In some cases, the parameters of the encoder and and the decoder are forced to be eqaul, i.e., $\mathcal{S}=\{\vc{\theta}=[\vc{\theta}_{\enc}, \vc{\theta}_{\dec}]|\vc{\theta}_{\enc} = \vc{\theta}_{\dec}\}$ \cite{vincent2010stacked}. 
 The constraint set can be $\reals^d$ (i.e., the problem is unconstrained)  or an $\ell_p$-norm ball (i.e., $\{\theta\mid \|\theta\|_p\leq r\},$ for some $r>0$, $p\geq 1$), when the magnitute of parameters is constrained; this can be used instead of a $\ell_1$ or $\ell_2$ norm regularizer. In stacked denoising autoencoders \cite{vincent2010stacked}, the encoder and decoder are shallow and satisfy the additional constraint $\theta_{\enc}=\theta_{\dec}$. %, i.e., parameters are equal. 

%Given $n$ data points $\vc{x}_i$, the goal is to embed them in a $m^\prime-$dimensional space ($m^\prime \ll m$). The embedding is done via a differentiable and possibly non-linear function $f_e:\reals^m\rightarrow \reals^{m^\prime}$. Another differentiable function $F_{\dec}:\reals^{m^\prime} \rightarrow \reals^ m$ is used to re-constructs the original points given the embeddings generated by  $f_e$. 
%For a data point $x_i$, the reconstruction  error is computed as
%$F_i(\vc{\theta}) = \vc{x}_i - f_d(f_e(\vc{x}_i)),$
%where $\vc{\theta}\in \reals^d$ is the vector of trainable parameters of $f_e$ and $f_e.$

%\item
\noindent \emph{Multi-target Regression   \cite{Spyromitros-Xioufis2016}.}
We are given a set of $n$ data points $\vc{x}_i\in\reals^{m}$, $i\in [n]$ and the corresponding target labels $\vc{y}_i \in \reals^{m^\prime}.$ 
The goal is to train a (again possibly non-linear) function $f:\reals^d\times \reals^m \to \reals^{m^\prime}$, with $d$ parameters, to predict target values for a given vector $\vc{x}\in \reals^m$. This maps to Prob.~\eqref{eq:problem} via: \begin{align}F(\vc{\theta};\vc{x}_i, \vc{y}_i)=\vc{y}_i - f(\vc{\theta}; \vc{x}_i).\end{align}
Robustness in this setting corresponds to ameliorating the effect of outliers in the \emph{label} space, i.e., among labels $\{\vc{y}_i\}_{i=1}^n$. The constraint set can again be $
\reals^d$ or defined through an $\ell_p$-norm ball (instead of the corresponding regularizer).

%based on a linear combination of a $N$-dimensional input values $\vc{x} \in \mathbb{R}^N$. Here, the loss function is defined as $F_i(\vc{\theta}) = \vc{y}_i - \langle \vc{x}_i,\vc{\theta} \rangle$, where $\vc{\theta} \in \mathbb{R}^{N \times M}$ is the matrix of trainable parameters.

%\item
\noindent\emph{Matrix Factorization   \cite{paatero1994positive}.} Given a matrix $\vc{X}\in \reals^{n \times m}$, the goal is to express it a the product of two matrices $\vc{G}$,$\vc{H}$. Cast in our setting,  each row  $\vc{x}_i \in \reals^m, i\in [n]$, of $\vc{X}$ is mapped to a lower dimensional sub-space as a vector $\vc{h}_i \in \reals^{m ^\prime}$, where the sub-space basis is defined by the rows of the matrix $\vc{G}\in \reals^{m \times m^\prime}$. %$\vc{x}_i$ are generated by the following assumption:
%  \begin{align*}
%     \vc{x}_i = \vc{G}\vc{h}_i + \vc{\epsilon},
% \end{align*}
% where $\vc{\epsilon}$ is a random additive noise. 
Function $F$ is then given by $F(\vc{\theta}; \vc{x}_i) =   \vc{x}_i - \vc{G}\vc{h}_i,$
where $\vc{\theta}=(\vc{G}, \vc{H})$ and   the rows of the matrix $\vc{H}\in \reals^{n \times m^\prime}$ are the low-dimensional embeddings $\vc{h}_i.$ Robustness here limits sensitivity to outliers in rows; a similar problem can be defined in terms of robustness to outliers in columns. Beyond usual boundedness constraints, additional constraints are introduced in so-called \emph{non-negative matrix factorization} \cite{paatero1994positive,fevotte2011algorithms}, where matrices $\vc{G}$ and $\vc{H}$ are  constrained to be non-negative.  
%\end{itemize}

For all three applications, we assume that  $F$ is smooth; this requires, e.g.,  smooth activation functions in deep models. Moreover, in all three examples, Prob.~\eqref{eq:problem} is  non-convex and non-smooth, as $\|\cdot\|_p$ is non-differentiable at $\mathbf{0}\in\reals^N$. 
%In particular, using the $\ell_1$ norm results in a regressor that is the median of the outputs sampled around an input point, wheres for the squared  $\ell_2$ norm, the estimator is the average of the outputs sampled around an input. Note that the median would remain unperturbed by adding few outliers, but the average can change dramatically. 
%The sensitivity of the squared $\ell_2$ norm to outliers and the effectiveness of using other norms for dealing with outliers, as explained above, motivates us to focus on solving \eqref{eq:problem}. 

%Specifically, we incorporate Online Alternating Direction Method of Multipliers (OADM) in to this framework.% This has multiple advantages compared to conventional (sub)gradient methods. 

%\todo{(should we mention challenges in sgd?)}
%(Other methods in non-smooth optimization should be mentioned in related works.)
%The challenge in dealing with non-smooth optimization problems such as \eqref{eq:problem} is that conventional gradient methods (e.g. SGD) often fail to solve them. 

\section{Robust Regression via MBO}\label{sct:MBO}
In this section, we outline how non-smooth, non-convex Prob.~\eqref{eq:problem} can be solved via \emph{model-based optimization} (MBO)~\cite{Ochs2019}. 
 MBO relies on the use of a  model function, which is a convex approximation of the main objective. In short, the algorithm proceeds iteratively, approximating function $F(\cdot;\cdot)$ by it's 1st order Taylor expansion at each iteration. This approximation is affine in $\vc{\theta}$, and results in a convex optimization problem per iteration.
 
In more detail, cast into our setting, MBO proceeds as follows. Starting with a feasible solution $\vc{\theta}^0\in \mathcal{C}$, it performs the following operations in  each step $k\in\naturals$:% %at the $k$-th iteration it solves  the following problem  and obtains a solution $\vc{\Tilde{\vc{\theta}}}^k$
\begin{subequations}%
\begin{align}\label{eq:unconstainedsub}%
   \tilde{\vc{\theta}}^{k}&= \argmin_{\vc{\theta}} F_{\vc{\theta}^k} (\vc{\theta}) + \frac{h}{2}\|\vc{\theta}- \vc{\theta}^k\|_2^2, \\
   \vc{\theta}^{k+1} &= (1-\eta^k) \vc{\theta}^k + \eta^k \Tilde{\vc{\theta}}^k,\label{eq:updstep}
\end{align}
\end{subequations}
where $h>0$ is a regularization parameter, $\eta^k>0$ is a step size, and function $F_{\vc{\theta}^k}:\reals^d \to \reals$ is the so-called \emph{model function} at $\vc{\theta}^k$, defined as:
\begin{align}\label{eq:MF}
    F_{{\vc{\theta}}^k}(\vc{\vc{\theta}})\triangleq  \frac{1}{n} \sum_{i\in [n]}\|F({\vc{\theta}^k}; \vc{x}_i)+ \vc{D} _{F_i}({\vc{\theta}^k})(\vc{\theta}-{\vc{\theta}}^k)\|_p +g(\vc{\theta}) + \chi_{\mathcal{C}}(\vc{\theta}), 
\end{align}
where $\vc{D}_{F_i}(\vc{\theta})\in \reals^{N\times d}$ is the Jacobian of  $F(\vc{\theta}; \vc{x}_i)$ w.r.t. $\vc{\theta}.$ Thus, in each step, MBO replaces $F$ with its 1st-order Taylor approximation and minimizes the objective   plus a proximal penalty; the resulting $\tilde{\vc{\theta}}^k$ is interpolated with the current solution $\vc{\theta}^k$.

The above steps are summarized  in Alg.~\ref{alg:mbo}.
  The step size $\eta^k$ is computed via an Armijo-type line search algorithm, which we present  in detail in \fullversion{App.~\ref{app:lsa} in  \cite{our}.}{App.~\ref{app:lsa}.}  Moreover, the inner-step optimization via \eqref{eq:unconstainedsub} can be inexact; the following proposition shows asymptotic convergence of MBO to a stationary point using an inexact solver (see also App.~\ref{app:lsa} \fullversion{in  \cite{our}}{}):
\begin{proposition} \label{prop:ochs} (Theorem 4.1 of \cite{Ochs2019})
Suppose $\vc{\theta}^*$ is the limit point of the sequence $\vc{\theta}^k$ generated by Alg.~\ref{alg:mbo}. Assume $F_{\vc{\theta}^k}(\Tilde{\vc{\theta}}^k) + \frac{h}{2}\|\Tilde{\vc{\theta}}^k- \vc{\theta}^k\|_2^2 - \inf_{\tilde{\vc{\theta}}} F_{\vc{\theta}^k}(\tilde{\vc{\theta}}) + \frac{h}{2}\|\tilde{\vc{\theta}}- \vc{\theta}^k\|_2^2 \leq \epsilon^k$, for all iterations $k$, and that $\epsilon^k \to 0$. Then $\vc{\theta}^*$ is a stationary point of Prob.~\eqref{eq:problem}.
\end{proposition}
For completeness, we prove Proposition~\ref{prop:ochs}  in Appendix~\ref{append:proofpropochs}, by showing that assumptions of Theorem 4.1 of \cite{Ochs2019} are indeed satisfied.  Problem~\eqref{eq:unconstainedsub} is convex but still non-smooth; we discuss how it can be solved efficiently via  SADM in the next section.

%The model based algorithm for constrained problems \cite{pmlr} solves the problems of the form
%\begin{align*}
%    \min_{\vc{\theta}\in \mathcal{C}} F(\vc{\theta}), 
%\end{align*}
%where $\mathcal{C}\subset \reals^d$ is non-empty convex compact set and  $F$ is a proper lower semi-continuous function that is bounded from below and its domain is a subset of $\mathcal{C}.$ The model functions here is exactly as defined in Def.~\ref{def:modelf}. In this case the subproblem \eqref{eq:unconstainedsub} is replaced with the following problem:
%\begin{align}
%    \label{eq:constrainedsub}
%     \vc{\Tilde{\vc{\theta}}}^k := \argmin_{\vc{\theta}\in \mathcal{C}} F_{\vc{\theta}^k} (\vc{\theta}). 
%\end{align}
%Again the problems can be solved inexactly, s.t., a solution $\Tilde{\vc{\theta}}^k$ only needs to improve the model function, i.e., $F_{\vc{\theta}^k}(\Tilde{\vc{\theta}}^k) - F_{\vc{\theta}^k}(\vc{\theta}^k)<0.$

\begin{algorithm}[!t] 
    \caption{Model-based Minimization (MBO)}\label{alg:MBO}
    \label{alg:mbo}
    \begin{algorithmic}[1]
    % The number tells where the line numbering should start
      \State {\bfseries Input: Initial solution $\vc{\theta}^0\in \dom F$, iteration number $K$ set $\delta, \gamma \in (0,1),$ and $\tilde{\eta} > 0$}
       \For{$k\in [K]$} 
       
         \State{ $\vc{\Tilde{\vc{\theta}}}^k := \argmin_{\vc{\theta}} F_{\vc{\theta}^k} (\vc{\theta}) + \frac{h}{2}\|\vc{\theta}- \vc{\theta}^k\|_2^2$}\label{alg_line:inner_prob}
         \State {Find $\gamma^k$ via Armijo search rule}
         \State{$\vc{\theta}^{k+1} := (1-\eta^k) \vc{\theta}^k + \eta^k \Tilde{\vc{\theta}}^k$}
       \EndFor
    \end{algorithmic}
\end{algorithm}

\section{Stochastic Alternating Direction Method of Multipliers}\label{sec:oadm}

After dealing with convexity via MBO, there are still two challenges behind solving the constituent sub-problem \eqref{eq:unconstainedsub}. The first is the non-smoothness of $\|\cdot\|_p$; the second is  scaling in $n$, which calls for a the use of a stochastic optimization method, akin to SGD (which, however, is not applicable due to the lack of smoothness). We address both through the a novel approach, namely, SADM, which is a stochastic version of the OADM algorithm by Wang and Banerjee \cite{online_admm, wang2013online}. Most importantly, our approach reduces the solution of Prob.~\eqref{eq:unconstainedsub} to several gradient-free optimization sub-steps, which can be computed  efficiently.  In addition, using an SADM/ADMM variant comes with clear stopping criteria, which is challenging for  traditional stochastic subgradient methods \cite{mai2020convergence}.

\subsection{SADM}
We first describe how our SADM %the standard ADMM \cite{boyd2011distributed} algorithm 
can be applied to solve Prob.~\eqref{eq:unconstainedsub}. %The major motivation for this choice is that as we show the subsequent ADMM steps admit efficient gradient-free solutions. Moreover, this gives us a clear stopping criteria for the inner problem which is challenging to find in traditional stochastic subgradient methods \cite{mai2020convergence}.  
 %Setting $\bar{\vc{\theta}}\triangleq \vc{\theta}^k$, 
 We  introduce the following notation to make our exposition more concise:%
\begin{subequations}%
\label{eq:newF}%
\begin{align}
F^{(k)}(\vc{\theta}; \vc{x}_i) &\triangleq\|F(\vc{\theta}^k; \vc{x}_i)+ \vc{D} _{F_i}(\vc{\theta}^k)(\vc{\theta}-\vc{\theta}^k)\|_p+ \frac{h}{2} \|\vc{\theta} - \vc{\theta}^k\|_2^2, \label{eq:newFrandom}\\  % +\chi_{\mathcal{C}}(\vc{\theta}) \\
    F^{(k)}(\vc{\theta})  &\triangleq \frac{1}{n}\sum_{i\in [n]}F^{(k)}(\vc{\theta}, \vc{x}_i),\\
    G(\vc{\theta})&\triangleq g(\vc{\theta}) + \chi_{\mathcal{C}}(\vc{\theta}).
\end{align}
\end{subequations}
 We can then rewrite Prob.~\eqref{eq:unconstainedsub} as the following equivalent problem:%
 \begin{subequations}\label{eq:admm_form}% 
\begin{align}%
\text{Minimize}&\quad F^{(k)}(\vc{\theta}_1) + G(\vc{\theta}_2) %\frac{1}{n} \sum_{i\in [n]}\|F(\bar{\vc{\theta}}; \vc{x}_i) + \vc{D} _{F_i}(\bar{\vc{\theta}})(\vc{\theta}-\bar{\vc{\theta}}))\|_p+ \frac{h}{2}\|\vc{\theta}_1 - \vc{\theta}^{(k)}\|_2^2 +  g(\vc{\theta}_2) +  \chi_{\mathcal{C}}(\vc{\theta}_2) 
\label{eq:caseaOBJ}\displaybreak[0]\\
\text{subject to:}&\quad \vc{\theta}_1 = \vc{\theta}_2,\label{eq:caseaCNST}  
\end{align}
\end{subequations}
where  $\vc{\theta}_1, \vc{\theta}_2\in \reals^d$ are auxiliary variables. %Note that $\bar{\vc{\theta}}$ is not an optimization variables, when running MBO, it is set to the iterations at the $k$-th iteration.  
%Moreover, for the case of strongly convex functions, as it is the case for some choices of the regularization, e.g., $g(\vc{\theta}) = \alpha \|\vc{\theta}\|_2^2$, the convergence rate is $O(\frac{\log T}{T})$, in comparison with $O(\frac{\log T}{T})$ \cite{shamir2013stochastic} for stochastic gradient methods. 
%The original ADMM algorithm for \eqref{eq:admm_form} has a convergence rate of $O(1/T)$ \cite{online_admm,he20121,monteiro2013iteration}; however, its computation complexity is linear with respect to $n$, e.g., number of training samples, which is typically in the order hundreds of thousands or millions. Therefore, we resort to using an online variant of ADMM proposed by Wang and Banerjee \cite{online_admm, wang2013online}; the authors called it Online Alternating Direction Method (OADM) and it updates variables using the loss corresponding to only one sample $\vc{x}_i$ at a time. The work by Wang and Banerjee considers a harder \emph{online} settings, where the samples are revealed by an \emph{adversary}. We adapt their results for the stochastic setting, where we choose the samples  uniformly at random. We show that in the this stochasticsetting and for strongly convex functions, OADM obtains a $O(\log T/T)$ convergence rate (see Theorem~\ref{trm:oadm}), which matches a similar result for stochastic gradient descent \cite{hazan2007logarithmic}. 

Note that the objective in \eqref{eq:caseaOBJ} is equivalent to $F^{(k)}(\vc{\theta}_1) + G(\vc{\theta}_2)$. 
SADM starts with initial solutions, i.e., $\vc{\theta}_1^{0}=\vc{\theta}_2^{0}=\vc{u}^0=0$. At the $t$-th iteration, the algorithm performs the following steps:
\begin{subequations}\label{eq:oadm}
\begin{align}
    \vc{\theta}_1^{t+1} :=& \argmin_{\vc{\theta_1}}F^{(k)}(\vc{\theta}_1; \vc{x}_t) + \frac{\rho_t}{2}\|\vc{\theta}_1 - \vc{\theta}_2^t +\vc{u}^t\|_2^2+ \frac{\gamma_t}{2} \|\vc{\theta}_1 - \vc{\theta}_1^t\|_2^2 \label{eq:update_theta1} ,\displaybreak[0]\\
    \vc{\theta}_2^{t+1} :=&\argmin_{\vc{\theta_2}} G(\vc{\theta}_2) + \frac{\rho_t}{2}\|\vc{\theta}^{t+1}_1 - \vc{\theta}_2 +\vc{u}^t\|_2^2  \label{eq:update_theta2},\displaybreak[0] \\
    \vc{u}^{t+1}:=& \vc{u}^{t} + \vc{\theta}_1^{t+1} - \vc{\theta}_2^{t+1},  \label{eq:update_dual}
\end{align}
\end{subequations}
where variables $\vc{x}_t$ are sampled  uniformly at random from $\{\vc{x}_i\}_{i=1}^n$,  $\vc{u}^t\in \reals^d$ is the dual variable, the $\rho_t,\gamma_t>0$ are scaling coefficients at the $t$-th iteration. We explain how to set $\rho_t,\gamma_t$ in Thm.~\ref{trm:oadm}. 

The solution to Problem \eqref{eq:update_theta2} amounts to finding the proximal operator of  function $G.$ In general, given that $g$ is smooth and convex, this is a strongly convex optimization problem and can be solved via standard techniques. Nevertheless, for several of the  practical cases we described in Sec.~\ref{sec:probapp} this optimization can be done efficiently with gradient-free methods. For example, in the case where the regularizer $g$ is the either a ridge or lasso penalty, and $\mathcal{C}=\reals^d$, it is well-known that proximal operators for  $\ell_1$ and $\ell_2$ norms have closed-form solutions \cite{boyd2011distributed}. For general $\ell_p$ norms, an efficient (gradient-free) bi-section method due to Liu and Ye \cite{liu2010efficient} (see App.~\ref{app:prox_op} \fullversion{in  \cite{our}}{}) can be used to compute the proximal operator. Moreover, in the absence of the  regularizer, the proximal operator for the indicator function $\chi_{\mathcal{C}}$ is equivalent to projection on the convex set $\mathcal{C}.$ This again has closed-form solution, e.g., when $\mathcal{C}$ is the simplex \cite{michelot1986finite} or an $\ell_p$-norm ball \cite{moreau1962decomposition,liu2010efficient}. 
%\begin{align}
%    \label{eq:theta2_sol}
%    \vc{\theta}_2^{t+1} = \frac{\rho}{\alpha + \rho}\left( \vc{\theta}_1^{t+1} + \vc{u}^t\right).
%\end{align}
Problem \eqref{eq:update_theta1} is harder to solve; %, as the objective is not differentiable. %In Sec.~\ref{sec:inner_ADMM}, we propose to reformulate this problem and solve it  via ADMM, again. 
%In particular, after running the OADM algorithm for  $T$ iterations, we output the following solution $\Tilde{\vc{\theta}}^k = \frac{1}{T}\sum_{t=1}^T \vc{\theta}_1^t;$ we provide theoretical guarantees for this  next. 
we show however that it can also reduced  to the (gradient-free) bisection method due to Liu and Ye \cite{liu2010efficient} in the next section.

\subsection{Inner ADMM}\label{sec:inner_ADMM}
%As mentioned in Section~\ref{sec:oadm}, 
 We solve Problem \eqref{eq:update_theta1} using another application of ADMM. In particular, note that \eqref{eq:update_theta1} assumes the following general form:
\begin{align}\label{eq:p_problem}
    \min_{\vc{x}}&\quad  \|\vc{A}\vc{x} + \vc{b}\|_p + \lambda \|\vc{x} - \vc{c}\|_2^2,
\end{align}
%where $\vc{A}_i\in \reals^{N\times d}, \vc{b}_i\in \reals^N, \vc{c}\in \reals^d$ and the optimization is w.r.t. $\vc{x}.$ It is easy to see that \eqref{eq:update_theta1} is an instance of this problem with
where  $\vc{A} = \vc{D}_{F_t}(\vc{\theta}^{(k)})$, the constituent parameter vectors are 
    $\vc{c} = \frac{\rho_t}{\rho_t + \gamma_t + h} (\vc{\theta}_2^t -\vc{u}^t) + \frac{\gamma_t}{\rho_t + \gamma_t +h} \vc{\theta}_1^t + \frac{h}{\rho_t + \gamma_t +h}\vc{\theta}^{(k)}, 
    \vc{b} = F(\vc{\theta}^{(k)};\vc{x}_t) - \vc{D}_{F_{t}}(\vc{\theta}^{(k)}) \vc{\theta}^{(k)},$ and
    $\lambda = \frac{\rho_t + \gamma_t +h}{2}.$

We solve \eqref{eq:p_problem} via ADMM by reformulating it as the following problem:
\begin{subequations}\label{eq:p_problem_re}
\begin{align}
\min&\quad \|\vc{y}\|_p + \lambda \|\vc{x} - \vc{c}\|_2^2\\
\text{s.t}&\quad \vc{A}\vc{x} + \vc{b} - \vc{y} = 0.
\end{align}
\end{subequations}
The ADMM steps at the $k$-th iteration for \eqref{eq:p_problem_re} are the following:
\begin{subequations}\label{eq:inner_ADMM}
\begin{align}
    \vc{y}^{k+1} &:= \argmin_{\vc{y}}  \|\vc{y}\|_p  +%\nonumber\\ &
    \rho'/2 \| \vc{y} - \vc{A}\vc{x}^k - \vc{b}  + \vc{z}^k\|_2^2,\label{eq:prox_p}\displaybreak[0]\\
    \vc{x}^{k+1} &:= \argmin_{\vc{x}} \lambda \|\vc{x} - \vc{c}\|_2^2 +%\nonumber\\&
    \rho'/2\| \vc{y}^{k+1} - \vc{A}\vc{x} - \vc{b}   + \vc{z}^k\|_2^2,\displaybreak[0]\\
    \vc{z}^{k+1}&:= \vc{z}^k + \vc{y}^{k+1} - \vc{A}\vc{x}^{k+1} - \vc{b},
\end{align}
\end{subequations}
where  $\vc{z}^k\in \reals^N$ denotes the dual variable at the $k$-th iteration and $\rho'>0$ is a hyper-parameter of ADMM. 

Problem~\eqref{eq:prox_p} is again equivalent to  computing the proximal operator of the $\ell_p$-norm, which, as mentioned earlier, has closed-form solution for $p=1, 2$. Moreover, for general $\ell_p$-norms the proximal operator can be computed via the bisection  algorithm
by Liu and Ye \cite{liu2010efficient}. This bisection method 
yields a solution with an $\epsilon$ accuracy in $O(\log_2(1/\epsilon))$ rounds  \cite{moharrer2020massively, liu2010efficient} (see App.~\ref{app:prox_op} \fullversion{in  \cite{our}}{}).

\subsection{Convergence}\label{sec:conv}
To attain the convergence guarantee of MBO given by Proposition~\ref{prop:ochs}, we need to solve the inner problem~\eqref{eq:model_improve} within accuracy $\epsilon^k$ at iteration $k$, where $\epsilon^k\to 0$. As our major technical contribution, we ensure this by proving the convergence of SADM  when solving  Prob.~\eqref{eq:model_improve}. %The original work  by Wang and Banerjee \cite{wang2013online} studied  the convergence of OADM under a different setting, where the samples $\vc{x}_t$ at the $t$-th iteration in \eqref{eq:update_theta1} are chosen by an \emph{adversary}. Their work shows that the difference between the summation over the functions revealed by the adversary for the solutions $\vc{\theta}_1^t$ and the optimum $\vc{\theta}^*$ is bounded by $O(\log T).$ However, we are interested in minimizing the average value $F^{(k)}(\vc{\theta})$; in the following theorem we show that, when the samples $\vc{x}_t$ are chosen uniformly at random, we obtain a similar result  for the average value. 

Consider the sequence $\{\vc{\theta}_1^t,\vc{\theta}_2^t, \vc{u}^t\}_{t=1}^T$ generated by our SADM algorithm \eqref{eq:oadm}, where $\vc{x}_t$, $t\in[T]$, are sampled u.a.r. from $\{\vc{x}_i\}_{i=1}^n.$ Let also \begin{align}\bar{\vc{\theta}}^T_1 \triangleq \frac{1}{T}\sum_{t=1}^T \vc{\theta}^t_1,~ \bar{\vc{\theta}}^T_2 \triangleq \frac{1}{T} \sum_{t=1}^{T}\vc{\theta}^{t+1}_2,\end{align} denote the time averages of the two solutions. Let also $\vc{\theta}^* = \vc{\theta}_1^* = \vc{\theta}_2^*$ be the optimal solution of Prob.~\eqref{eq:admm_form}.
 Finally, denote by \begin{align}R^T\triangleq F^{(k)}(\bar{\vc{\theta}}^T_1) + G(\bar{\vc{\theta}}^T_2)- F^{(k)}(\vc{\theta}^*)- G(\vc{\theta}^*)\end{align}
 the residual error of the objective from the optimal.
 Then, the following holds:
\begin{theorem}\label{trm:oadm}
Assume that $\mathcal{C}$ is convex, closed, and bounded, while $g(0) = 0$, $g(\vc{\theta})\geq 0$, and $g(\cdot)$ is both Lipschitz continuous and $\beta$-strongly convex over $\mathcal{C}$. Moreover, assume that both the function $F(\vc{\theta}; \vc{x}_i)$ and its Jacobian $\vc{D}_{F_i}(\vc{\theta})$ are bounded on the set $\mathcal{C}$, for all $i\in[n]$.  We set $\gamma_t=ht$ and  $\rho_t=\beta t$.  Then,%
\begin{subequations}\label{eq:oadm_conv}%
\begin{align}%
    \label{eq:feas_eq}
\|\bar{\vc{\theta}}^{T}_1 - \bar{\vc{\theta}}^{T}_2\|_2^2 &= O\left(\frac{\log T}{T}\right)\quad\quad\quad\\
\E [ R^T  ]&= O\left(\frac{\log T}{T}\right) \label{eq:oadm_opt}\\
  \mathbb{P}\left( R^{T} \geq  k_1 \frac{\log T}{T}   + k_2 \frac{M}{\sqrt{T}}\right) &\leq e^{-\frac{M^2}{16}}~~\text{for all}~M>0, T\geq 3, \label{eq:oadm_opt_bound}
\end{align}
\end{subequations}
%shows the distance of the current objective from optimal. 
where $k_1, k_2>0$ are constants  (see \eqref{eq:k1k2k3} in App.~\ref{append:proofoadm} \fullversion{in  \cite{our}}{} for exact definitions). 
\end{theorem}

We prove Theorem~\ref{trm:oadm} in Appendix~\ref{append:proofoadm}. The theorem has the following important consequences. First, \eqref{eq:feas_eq} implies that the infeasibility gap between $\theta_1$ and $\theta_2$ decreases as $O(\frac{\log T}{T})$ \emph{deterministically}. Second, by \eqref{eq:oadm_opt} the residual error $R^T$ decreases as $O(\frac{\log T}{T})$ in expectation. Finally, 
%The main consequence of the theorem  is that   both the in-feasibility and the  expected value of sub-optimality (error)  decrease  with the order of $O(\frac{\log T}{T})$, due to \eqref{eq:feas_eq} and \eqref{eq:oadm_opt}, respectively. In addition, 
~\eqref{eq:oadm_opt_bound} shows that the tail of the residual error as iterations increase is exponentially bounded. In particular, given a desirable accuracy $\epsilon_k$, \eqref{eq:oadm_opt_bound} gives the number of iterations necessary be within $\epsilon_k$ of the optimal with any probability $1-\delta$. Therefore, according to Proposition~\ref{prop:ochs}, using SADM will result in convergence of Algorithm~\ref{alg:mbo} with high probability.
Finally, we note that, although we write Theorem~\ref{trm:oadm} for updates using only one random sample per iteration, the analysis and guarantees readily extend to the case where a batch selected u.a.r. is used instead. A formal statement and proof can be found in \fullversion{App.~\ref{append:corbatch} in  \cite{our}.}{App.~\ref{append:corbatch}.} 

\section{Experiments}\label{sec:experiments}
\begin{table}[t!]
    \centering
    \caption{Time and Objective Performance. We report objective and time metrics for under different outlier ratios and different $p$-norms. We observe from the table that \mboa{} significantly outperforms other competitors in terms of objective metrics. In terms of running time, \sgd{} is generally fastest, due to fast gradient updates. However, we see that the time MBO variants take to get to the same or better objective value (i.e., $T^*$), ware comparable to running time of \sgd{}. }\label{tab:metrics}
    \begin{tiny}
    \begin{tabular}{|c|c|c|c|c|c|c||c|c|c|c|c||c|c|c|c|}\hline
        &   & \multicolumn{5}{|c||}{\mboa{}} &\multicolumn{5}{c||}{\mbob{}}& \multicolumn{4}{c|}{\sgd{}}  \\\hline
        $\outl$ & $p$ &   $F_{\texttt{NOUTL}}$ & $F_{\texttt{OBJ}}$&  $F_{\texttt{TEST}}$ &$T$(h)& $T^*$(h) & $F_{\texttt{NOUTL}}$ & $F_{\texttt{OBJ}}$&  $F_{\texttt{TEST}}$ &$T$(h)& $T^*$(h)  & $F_{\texttt{NOUTL}}$ & $F_{\texttt{OBJ}}$&  $F_{\texttt{TEST}}$ &$T$(h)\\\hline
         \multicolumn{16}{|c|}{\texttt{MNIST}}\\\hline
         0.0 & 2.0 & \textbf{2.50} & \textbf{2.51} & \textbf{2.50} & 5.69   &   0.14 & 8.08 &  8.08 & 8.12 &64.17 & 6.47   &9.21 &  9.22 & 9.30 & 9.83 \\
         0.0 & 1.5 &  \textbf{2.63} & \textbf{2.63} & \textbf{2.63}  &11.67  &0.79  & 20.19& 20.20 & 20.39& 65.67& 59.98 & 20.35& 20.36 & 20.57 & 14.71 \\
         0.0 & 1.0 & \textbf{3.46} & \textbf{3.47} & \textbf{3.44} & 17.82 & 3.09 & 102.79 & 102.80 & 104.24 & 81.53 &NA &102.44 & 102.46 & 103.89 & 11.50 \\
      %   0.0 & $\ell_2^2$ & & & & & & & & & & & 87.81 & 87.82 & 89.34 & 13.95\\
         0.05 & 2.0 & \textbf{3.48} & \textbf{5.35}  & \textbf{3.46} & 6.33 &0.31  & 3.89& 6.36 & 3.86 & 54.92 &   38.31  &8.03& 12.52 & 8.09& 13.96 \\
         0.05 & 1.5 &  \textbf{4.10}&  \textbf{9.74} & \textbf{4.08}&45.32 &2.08  & {5.86} &{11.70} & {5.82} & 57.03&       25.12& 20.34 & 34.69& 20.57& 14.60 \\
          0.05 & 1.0 &\textbf{5.23} &  \textbf{20.20}& \textbf{5.24}&  44.03& 5.61 & {27.68} & {73.53} & {27.56} & 32.76&  9.67 &102.40& 236.43 & 103.90 & 11.70\\
        %  0.05 & $\ell_2^2$ & & & & && & & & & & 87.01& 417.84 & 86.68 & 10.24\\
        0.1 & 2.0 & 4.27& \textbf{7.77} &4.23 & 11.67  & 1.34 & \textbf{3.56}& {7.83} & \textbf{3.54} & 64.20& 33.70 &7.02& 11.64 & 7.04 & 13.97 \\
          0.1 & 1.5 & \textbf{4.18} & \textbf{11.84} & \textbf{4.17}& 68.74&0.29 &{5.50} & {13.77} & {5.45} & 67.04 &      9.88 &20.34& 48.79 & 20.57 & 14.04\\
          0.1 & 1.0 & \textbf{5.90} & \textbf{36.02} & \textbf{5.92}& 37.73&  8.20& {30.08}& {109.79}& {30.16}& 39.77 &   6.72   &102.36& 368.22& 103.90 & 11.81\\
        %  0.1 & $\ell_2^2$ & & & & & & & & & & &188.54& 708.47  & 187.76 & 14.24\\
        0.2 & 2.0 & {4.07} & 8.97 &  {4.04} & 51.69 & 4.39 &\textbf{3.54}&  \textbf{8.23} &  \textbf{3.52} & 57.08 & 19.19 &7.48 &  16.44 & 7.51 & 14.25\\
          0.2 & 1.5 & \textbf{3.90} & \textbf{11.58} & \textbf{3.89}&195.69 & 1.56&{7.00} &  {20.63} & {6.95} & 45.46 &6.44& 20.36&  77.78 &20.59 & 15.06 \\
          0.2 & 1.0&  \textbf{3.85} & \textbf{28.25}& \textbf{3.83}& 36.98&5.15 & {40.12}& {224.47} & {40.11} & 19.71&2.13 &  102.37& 639.32 & 103.90 & 8.25\\
        %    0.2 & $\ell_2^2$ & & & & & & & & & 624.09& 1542.18 &  622.55 & 13.35\\
        0.3 & 2.0 & \textbf{3.99} & {14.21} & \textbf{3.98} &  9.83& 4.93  & 4.02&  \textbf{10.55} & 3.99 & 55.60 & 24.14  &7.46& 20.63 & 7.48 & 13.53\\
            0.3 & 1.5 & 20.55 & \textbf{22.56} &20.78&159.92 & 39.24 & \textbf{7.22} & {24.30} &\textbf{7.16}& 42.65 &    15.09 &23.90 & 58.52 &  23.89 & 16.25\\
            0.3 & 1.0 & 102.70 & \textbf{99.36} & 104.27& 51.48& 0.87&\textbf{56.60}& {438.89} &  \textbf{56.17} & 20.48 &  3.32  &102.34 & 910.68 & 103.90 & 8.52\\\hline
          %  0.3 & $\ell_2^2$ & & & & & & & & &744.11& 2106.19 & 741.99 & 12.45\\hline
            \multicolumn{16}{|c|}{\texttt{Fashion-MNIST}}\\\hline
            0.0 & 2.0&  \textbf{3.51}& \textbf{3.51} & \textbf{3.51} &4.33 & 0.31 & 5.01&5.01 & 5.01 & 42.13 & 14.80& 8.72& 8.73 & 8.70 & 9.78 \\
            0.0 & 1.5 & \textbf{6.13} & \textbf{6.14} &  \textbf{6.14} &14.35 &2.18 & 8.87 & 8.88 & 8.89 &63.39& 12.80& 22.62& 22.63 & 22.56 & 14.70\\
            0.0 & 1.0 & \textbf{10.59} & \textbf{10.61}& \textbf{10.56}  & 29.69&2.63 &{41.24} & {41.26} & {41.35} & 50.41 &3.32 &224.26 & 224.28 & 224.89 & 9.72\\
             0.05 & 2.0 &\textbf{3.80} & \textbf{5.82} & \textbf{3.80} & 14.70&  1.40 & 4.53 & 6.75 & 4.54 & 67.29 &19.15& 8.30 & 11.98 & 8.27 & 9.71 \\
            0.05 & 1.5 &   \textbf{7.38}&  14.57 & \textbf{7.40}& 96.73 &2.56  &7.91 & \textbf{11.97} & {7.93} & 64.25 &  16.16  &  20.88 & 27.22 &  20.83 & 10.94\\
            0.05 & 1.0 &\textbf{16.64}& \textbf{30.51} & \textbf{16.68} &43.55&16.04&{65.01}& 109.94 & 65.31 & 29.60 &9.70& 158.65 & 227.48 & 158.27 & 9.97 \\
            0.1 & 2.0 &\textbf{4.05} & \textbf{6.73 } & \textbf{4.06}  & 14.66&2.35 &4.28 &  7.90 & 4.29 & 65.06 &16.46& 8.96 & 14.35 &8.94 & 13.67\\
            0.1 & 1.5 & 11.08 & 32.69 & 11.10 &20.07&NA &\textbf{8.46} & \textbf{15.23} &  \textbf{8.49} & 65.78&9.92 &  17.98 &  31.41 & 17.95 & 10.63\\
            0.1 & 1.0 & \textbf{9.79}& \textbf{27.96} & \textbf{9.81}& 35.50 &2.43 & {58.70} & {126.18} & {58.90} & 45.06 & 2.08 & 235.02 & 452.45 &  234.49 & 13.32\\
            0.2 & 2.0 & {6.07} & {10.19} & {6.08} &14.25 &3.67 & \textbf{4.77} &  \textbf{9.28} &  \textbf{4.77} & 69.84& 30.16 &  5.71&14.34 & 5.71 & 9.31\\
            0.2 & 1.5 & 28.51 & 53.69 & 28.49& 39.42& NA&\textbf{10.94} & \textbf{27.12} & \textbf{10.97} & 39.11 & 16.09 & 19.36& 42.97 & 19.36 & 10.87\\
            0.2 & 1.0 & \textbf{10.50}& \textbf{27.95} & \textbf{10.50}& 94.72 &6.57 &{140.00}&  390.02 & {140.08} & 17.03 & 3.33 & 204.88 &  644.99 &  205.13 & 14.72 \\
            0.3 & 2.0 & 6.63 & 23.18& 6.63 & 32.87&  NA & \textbf{5.84} & \textbf{13.04} &  \textbf{5.85} & 50.52&  29.95 & 7.45 &  20.12 & 7.46 & 13.65\\
            0.3  & 1.5 &\textbf{7.08} & \textbf{22.51} & \textbf{7.10}&86.27&30.02 &{11.09} & {24.73} & {11.12} & 52.41 &  12.08 &19.52& 58.26 &  19.56 & 11.05\\
            0.3 & 1.0 & \textbf{14.43} &  \textbf{50.91} & \textbf{14.46} & 95.08&19.48 &404.77 & 893.56 &  404.52 & 9.51 &NA & {410.82}& {522.50} & {411.84} & 10.74\\
            \hline
            \multicolumn{16}{|c|}{\texttt{SCMD1d}}\\\hline
            0.0 & 2.0 & 2.88 & 2.88 &  3.02 &1.82 &0.12 &  \textbf{2.85} & \textbf{2.85} & \textbf{2.99} &0.36 & 0.04&  3.62 & 3.63 & 3.72 & 1.37\\
            0.0 & 1.5 & 4.23 & 4.24 &  \textbf{4.39} & 7.22& 0.43&  \textbf{4.22} & \textbf{4.23} & 4.44 &0.36 & 0.04&5.47& 5.47& 5.60 & 1.58\\
             0.0 & 1.0 & \textbf{9.78} &  \textbf{9.79} & \textbf{10.18} &7.13 &0.47 & 9.86& 9.86 & 10.32 &0.37 &0.04 & 12.95 & 12.95 & 13.25 &1.22 \\
               0.05 & 2.0 & 2.88 & 3.13 &  2.99 &2.52 &0.13 & \textbf{2.86} & \textbf{3.11} & 3.00 & 0.54&0.05 & 3.64 & 3.89 &  3.71 & 1.31\\
             0.05 & 1.5 &  4.23 & 4.61 & \textbf{4.37} & 10.23& 4.61& \textbf{4.22} & \textbf{4.59} & 4.46 &0.50 & 0.05& 5.50& 5.87&  5.61 &1.23\\
           0.05 & 1.0 &  \textbf{9.69} & \textbf{10.52} & \textbf{10.09} & 0.59& 0.18&  9.86& 10.66 &10.35 &0.51 &0.05 &13.03& 13.87 & 13.29 &1.17\\
           0.1 & 2.0 & 2.90 & 3.41 &   3.01 &2.22 &0.13 & \textbf{2.84} & \textbf{3.34} & \textbf{3.00} & 0.46& 0.05& 3.63 & 4.12 & 3.69 &1.30 \\
           0.1 & 1.5 &  4.23 & 4.99 & 4.42 & 9.42& 0.68&\textbf{4.18} & \textbf{4.90} & \textbf{4.40} &0.50 &0.05 & 5.52 &6.11 &  5.62 &1.15 \\
             0.1 & 1.0 & \textbf{9.56} & \textbf{10.99} & \textbf{10.18} & 9.92&0.78 & 9.77 &  11.11 &  10.54 &0.54 & 0.07&13.09 &  13.72& 13.32 & 1.11\\
             0.2 & 2.0 & 2.93 & 3.90 &  3.03 &1.83 &0.95 & \textbf{2.86} & \textbf{3.79} &  \textbf{3.02} & 0.5& 0.3& 3.63 &  3.97 & 3.66 & 1.15\\
             0.2 & 1.5 & 4.23 &  5.60 &  \textbf{4.37} &8.17 & 3.60& \textbf{4.21} & \textbf{5.48} &  4.47 & 0.36& 0.2& 5.50 & 5.83 & 5.56 &1.17\\
             0.2 & 1.0 &  \textbf{9.46} & \textbf{11.00} &  \textbf{10.04} &6.55 & 1.50& 9.82 &  11.30 &  10.60 &0.45 & 0.20& 13.09& 13.38&  13.28 &1.11\\
             0.3 & 2.0 & 2.93 &  4.32 & 3.03 &1.80 & NA& \textbf{2.85} & 4.10 &\textbf{3.03} & 0.46& NA& 3.61 & \textbf{3.90}& 3.64 & 1.18\\
             0.3 & 1.5 & 4.25 & 6.05 & \textbf{4.44} &8.19 & NA& \textbf{4.21} & 5.73 & 4.49 &0.50 &NA & 5.43 & \textbf{5.69}& 5.43 &1.18 \\
              0.3 & 1.0 & \textbf{9.53}& 11.07 & \textbf{10.00} &6.44 & 2.72& 9.68&\textbf{11.05} &10.32&0.51 &0.28 & 12.95&13.13 & 12.96& 1.11\\
             \hline
    \end{tabular}
    \end{tiny}
    \fullversion{\vspace{-5mm}}{\vspace{-5mm}}
\end{table}

\noindent \textbf{Algorithms.} We run two variants of MBO; the first one, which we call \mboa{}, uses SADM (see Sec.~\ref{sec:oadm}) for solving the inner problems \eqref{eq:unconstainedsub}. The  second one, which we call \mbob{}, solves inner problems via a sub-gradient method. We also apply stochastic gradient descent with momentum directly to Prob.~\eqref{eq:problem}; we refer to this algorithm as \sgd{}. This corresponds to the algorithm by \cite{mai2020convergence}, applied to our setting. We also solve the problem instances with an MSE objective using \sgd{}, as the MSE is smooth and \sgd{} is efficient in this case.   Hyperparameters and implementation details are in \fullversion{App.~\ref{app:algodet} in  \cite{our}.}{App.~\ref{app:algodet}.} Our code is publicly available.\footnote{\href{https://github.com/neu-spiral/ModelBasedOptimization}{https://github.com/neu-spiral/ModelBasedOptimization}}

\noindent \textbf{Applications and Datasets.} We focus on two applications: training autoencoders and multi-target regression, with a ridge regularizer and $\mathcal{C}=\reals^d$. The architectures we use are described in \fullversion{App.~\ref{app:algodet} in  \cite{our}.}{App.~\ref{app:algodet}.} For autoencoders, We use \texttt{MNIST} and \texttt{Fashion-MNIST} to train autoencoders and \texttt{SCM1d}  \cite{Spyromitros-Xioufis2016} for multi target regression. All three datasets, including training and test splits, are described in \fullversion{App.~\ref{app:algodet} in  \cite{our}.}{App.~\ref{app:algodet}.}

\noindent \textbf{Outliers.} We denote the outliers ratio with $\outl$; each datapoint $\vc{x}_i$, $i\in[n]$, is independently corrupted with outliers with probability $\outl.$ The probability $\outl$ ranges from 0.0 to 0.3 in our experiments. In particular, we corrupt training samples by replacing them with  samples randomly drawn from a Gaussian distribution whose mean is  $\dc$ away from the original data and its standard deviation equals that of the original dataset. For \texttt{MNIST} and \texttt{FashionMNIST}, we set $\dc$ to 1.5 times the original standard deviation, while  for \texttt{SCM1d}, we set $\dc$ to 2.5 times the standard deviation. %We denote the set of points that are outliers as $\mathcal{S}_{\texttt{OUTL}}\in 2^{[n]}.$

\noindent\textbf{Metrics.} We evaluate the solution obtained by different algorithms by using the following three metrics. The first is $F_{\texttt{OBJ}}$, the regularized objective of Prob.~\eqref{eq:problem} evaluated over the training set. The other two are: %
%\begin{align}%
    $F_{\texttt{NOUTL}}\triangleq \frac{\sum_{i\notin\mathcal{S}_{\texttt{OUTL}}}\|F(\vc{\theta}; \vc{x}_i)\|_p }{n-|\mathcal{S}_{\texttt{OUTL}}|},$ and
    %&
   $ F_{\texttt{TEST}}\triangleq \frac{\sum_{i\in\mathcal{S}_{\texttt{TEST}}}\|F(\vc{\theta}; \vc{x}_i)\|_p}{|\mathcal{S}_{\texttt{TEST}}|},$
%\end{align}
where $\mathcal{S}_{\texttt{OUTL}}$, $\mathcal{S}_{\texttt{TEST}}$ are the outlier and test sets, respectively.
%The first, is the objective function $F_{\texttt{OBJ}}\triangleq \frac{1}{n}\sum_{i\in[n]}\|F(\vc{\theta}; \vc{x}_i)\|_p + g(\vc{\theta})$. We also report the average value of loss terms for both non-outlier points as well as the test points, i.e.,  
%  $F_{\texttt{NOUTL}}\triangleq \frac{1}{n-|\mathcal{S}_{\texttt{OUTL}}|}\sum_{i\notin\mathcal{S}_{\texttt{OUTL}}}\|F(\vc{\theta}; \vc{x}_i)\|_p$ and 
%  $F_{\texttt{TEST}}\triangleq \frac{1}{|\mathcal{S}_{\texttt{TEST}}|}\sum_{i\in\mathcal{S}_{\texttt{TEST}}}\|F(\vc{\theta}; \vc{x}_i)\|_p,$
%respectively. $F_{\texttt{OBJ}}$ shows the efficiency of algorithms in terms of minimizing the objective function; note that since the considered problems are non-convex the algorithms do not necessarily converge to the same value. 
Metric $F_{\texttt{NOUTL}}$ measures the robustness of algorithms w.r.t. outliers; ideally, $F_{\texttt{NOUTL}}$ should remain unchanged as the fraction of outliers increases. Metric $F_{\texttt{TEST}}$ evaluates the generalization ability of algorithms on unseen (test) data, which also does not contain outliers; ideally, $F_{\texttt{TEST}}$ be similar  $F_{\texttt{NOUTL}}$. %We refer to these three metrics as \emph{objective metrics}.
Moreover, we  report total running time ($T$) of all algorithms. For the two variants of  MBO, we additionally report the time ($T^*$) until the they reach the optimal value attained by \sgd{} (N/A if never reached). %. For cases that theses algorithms do not achieve objective values better than \sgd{}, we report NA for $T^*$. We refer to these metrics as \emph{time metrics}.
Finally, for autoencoders, we also use dataset labels to train a logistic regression classifier over latent embeddings, and also report the prediction accuracy on the test set. Classifier hyperparameters  are described in \fullversion{App.~\ref{app:classification} in  \cite{our}.}{App.~\ref{app:classification}.}

\subsection{Time and Objective Performance  Comparison}\label{sec:time_obj_pfm}
 We  evaluate our algorithms w.r.t. both objective and time metrics, which we report for different outlier ratios $\outl$ and  $p$-norms in Table~\ref{tab:metrics}. %We next discuss main observations and  takeaways from the Table.
 By comparing objective metrics, we see that  \mboa{} and \mbob{} significantly outperform \sgd{}. \sgd{} achieves a better $F_{\texttt{OBJ}}$ in only 2 out of 45 cases, i.e., \texttt{SCM1d} dataset for $p=1.5, 2$ and $\outl=0.3$; however, even for these two cases, \mboa{} and \mbob{} obtain better $F_{\texttt{NOUTL}}$ and $F_{\texttt{TEST}}$ values. In terms of overall running time $T$,  \sgd{} is generally faster than \mboa{} and \mbob{}; this is expected, as each iteration of \sgd{} only computes the gradient of a mini-batch of terms in the objective, while the other methods need to solve an inner-problem. Nonetheless, by comparing $T^*$, we see that the MBO variants obtain the same or better objective as \sgd{} in a comparable time. In particular, $T^*$ is less than $T$ for \sgd{} in 33 and 15 cases (out of 45) for \mboa{} and \mbob{}, respectively.

Comparing the performance between \mboa{} and \mbob{}, we first note that \mboa{} has a superior performance w.r.t.~all three objective metrics for 25 out of 45 cases. In some cases, \mboa{} obtains considerably smaller objective values; for example, for \texttt{MNIST} and $\outl=0.0, p=1,$ $F_{\texttt{NOUT}}$ is 0.03 of the value obtained by \mbob{} (also see Figures~\ref{fig:fout_scale_1} and \ref{fig:ftest_scale_1}). However, it seems that in the high-outlier setting $\outl=0.3$ the performance of \mboa{} deteriorates; this is mostly due to the fact that the high number of outliers adversely affects the convergence of SADM and it takes more iterations to satisfy the desired accuracy.

%\begin{figure}[t!]
%    \subfloat[$p=1,\outl=0.0$]{\includegraphics[width=0.33\textwidth]{plots/MNIST_outl00_p1_time.pdf}}
%    \subfloat[$p=1.5,\outl=0.0$]{\includegraphics[width=0.33\textwidth]{plots/MNIST_outl00_p1.5_time.pdf}}
%    \subfloat[$p=2,\outl=0.0$]{\includegraphics[width=0.33\textwidth]{plots/MNIST_outl00_p2_time.pdf}}
%    \caption{Convergence of different algorithms for $\outl=0.0$ and \texttt{MNIST}. We observe that \sgd{} achieves  inferior local minima and the objective increases after a number of iterations. In contrast, MBO variants obtain significantly lower objectives. In these cases, \mboa{} converges faster than \mbob{}.}\label{fig:conv}
%\end{figure}

%\subsection{Convergence}
%As we mention in Sec.~\ref{sec:time_obj_pfm}, the running times for \mboa{} and \mbob{} are higher than \sgd{}, though they result in significantly better objective metrics. Here, we show convergence of different algorithms to provide more insights. In particular, we show the trace of the objective $F_{\texttt{OBJ}}$ w.r.t. running time for \texttt{MNIST} and different $p$ values in Fig.~\ref{fig:conv}. In these figures, we run \sgd{} for more ($10^5$) iterations to show that \sgd{} \emph{does not} achieve a better objective by running longer. We indeed see that the objectives increase after a  number of iterations; this suggests that stochastic gradient methods \emph{fail} in minimizing these non-smooth non-convex objectives.  %\todo{Provide citations on SGD failure.}

\begin{figure}[t!]
    \subfloat[Non-outliers Loss, MSE]{\includegraphics[width=0.33\textwidth]{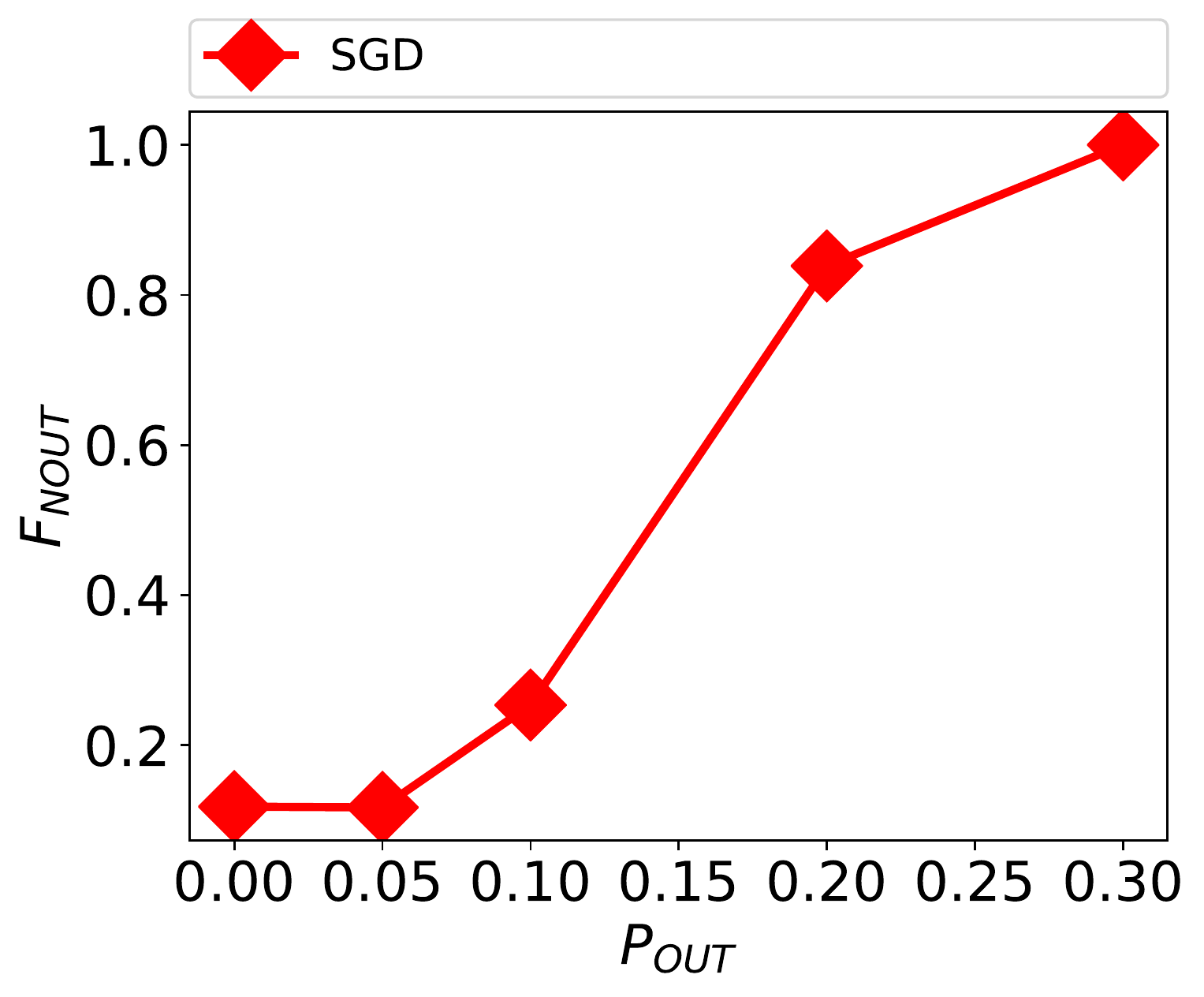}\label{fig:fout_scale_sq}}
    \subfloat[Non-outliers Loss, $p=2$]{\includegraphics[width=0.33\textwidth]{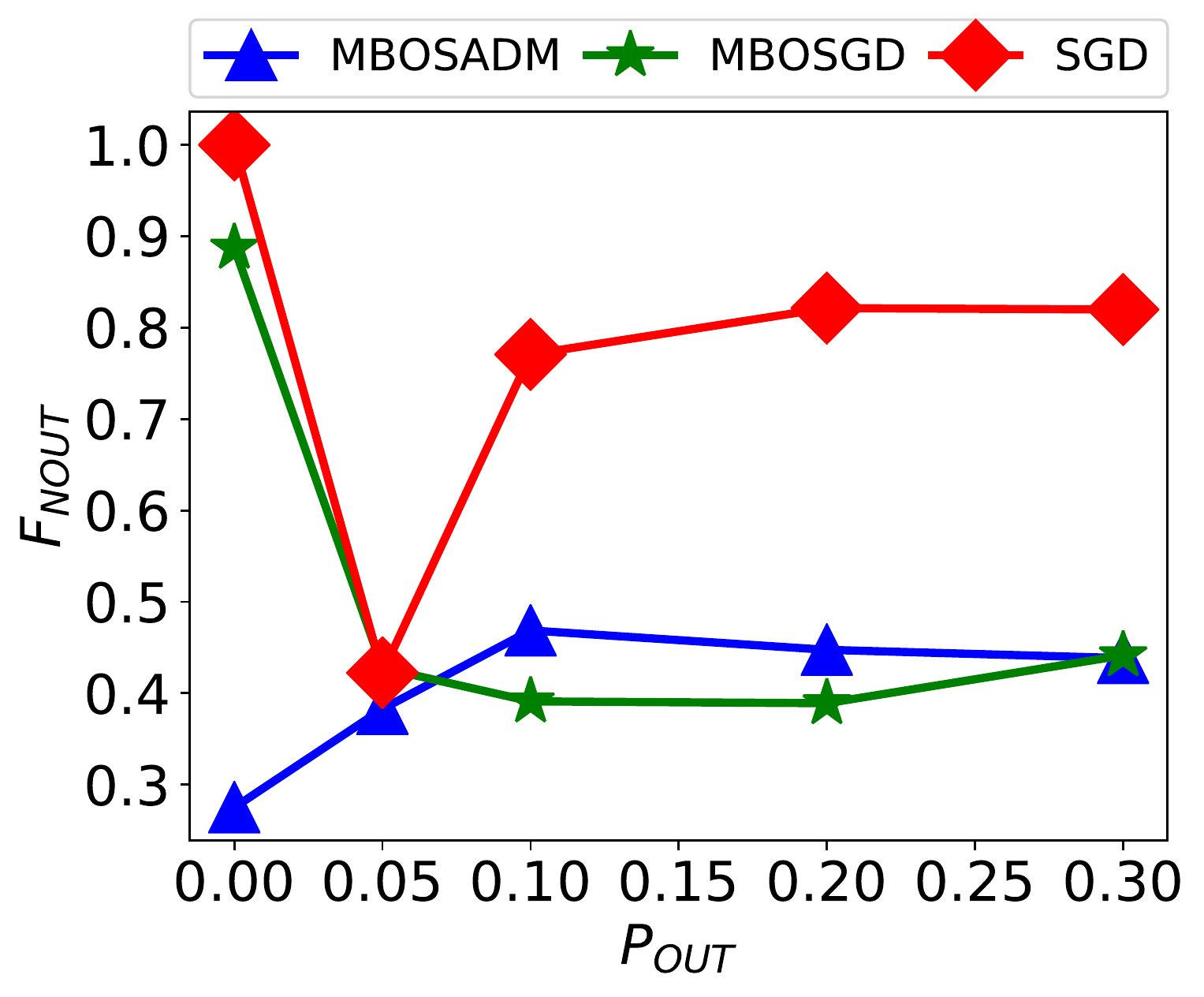}\label{fig:fout_scale_2}}
    %\subfloat[$p=1.5$]{\includegraphics[width=0.25\textwidth]{plots/linePlot_MNIST_p1.5_NOUT.pdf}\label{fig:fout_scale_1.5}}
    \subfloat[Non-outliers Loss, $p=1$]{\includegraphics[width=0.33\textwidth]{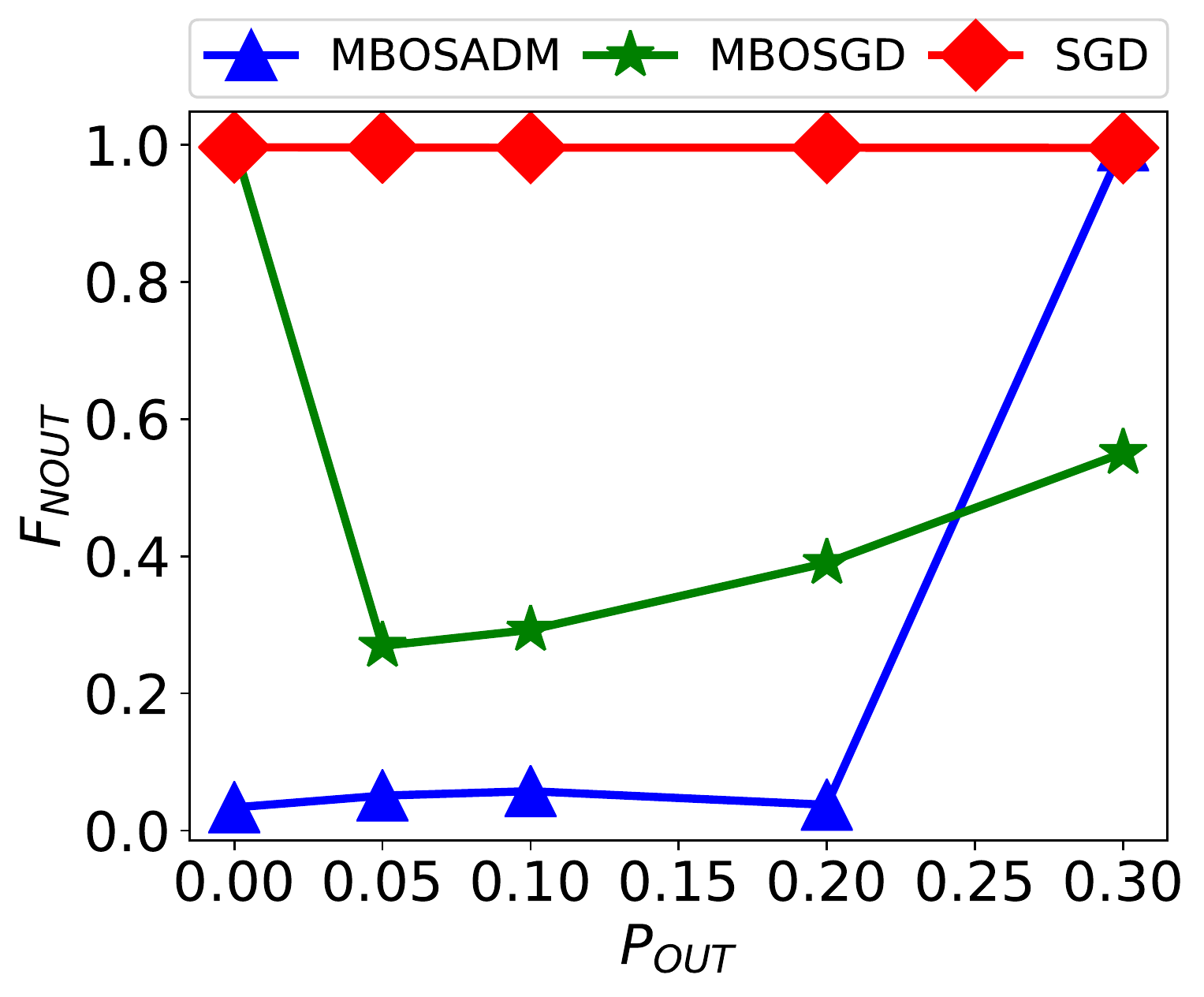}\label{fig:fout_scale_1}}
    
    \vspace{-3mm}
    \subfloat[Test Loss, MSE]{\includegraphics[width=0.33\textwidth]{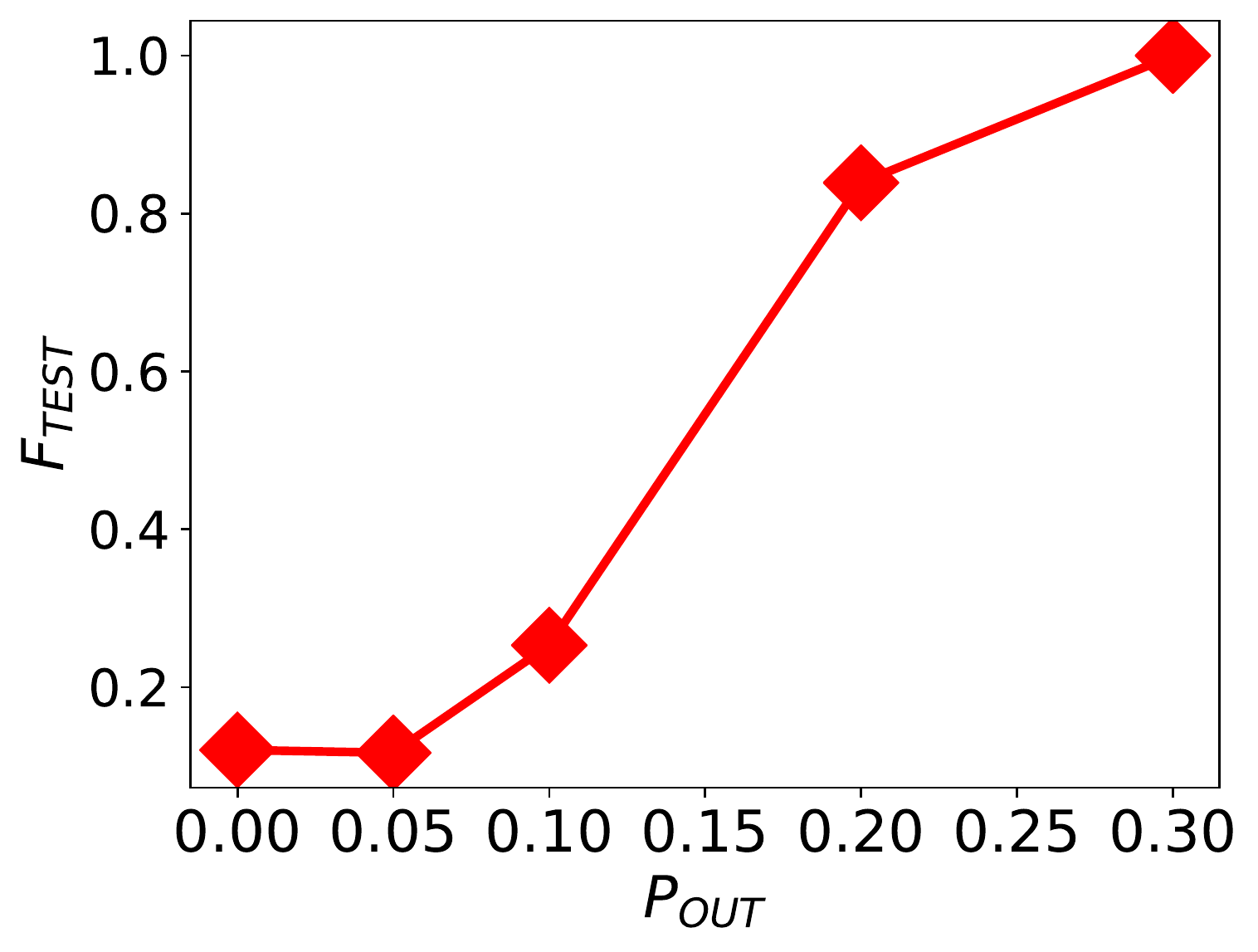}\label{fig:ftest_scale_ell2}}
    \subfloat[Test Loss, $p=2$]{\includegraphics[width=0.33\textwidth]{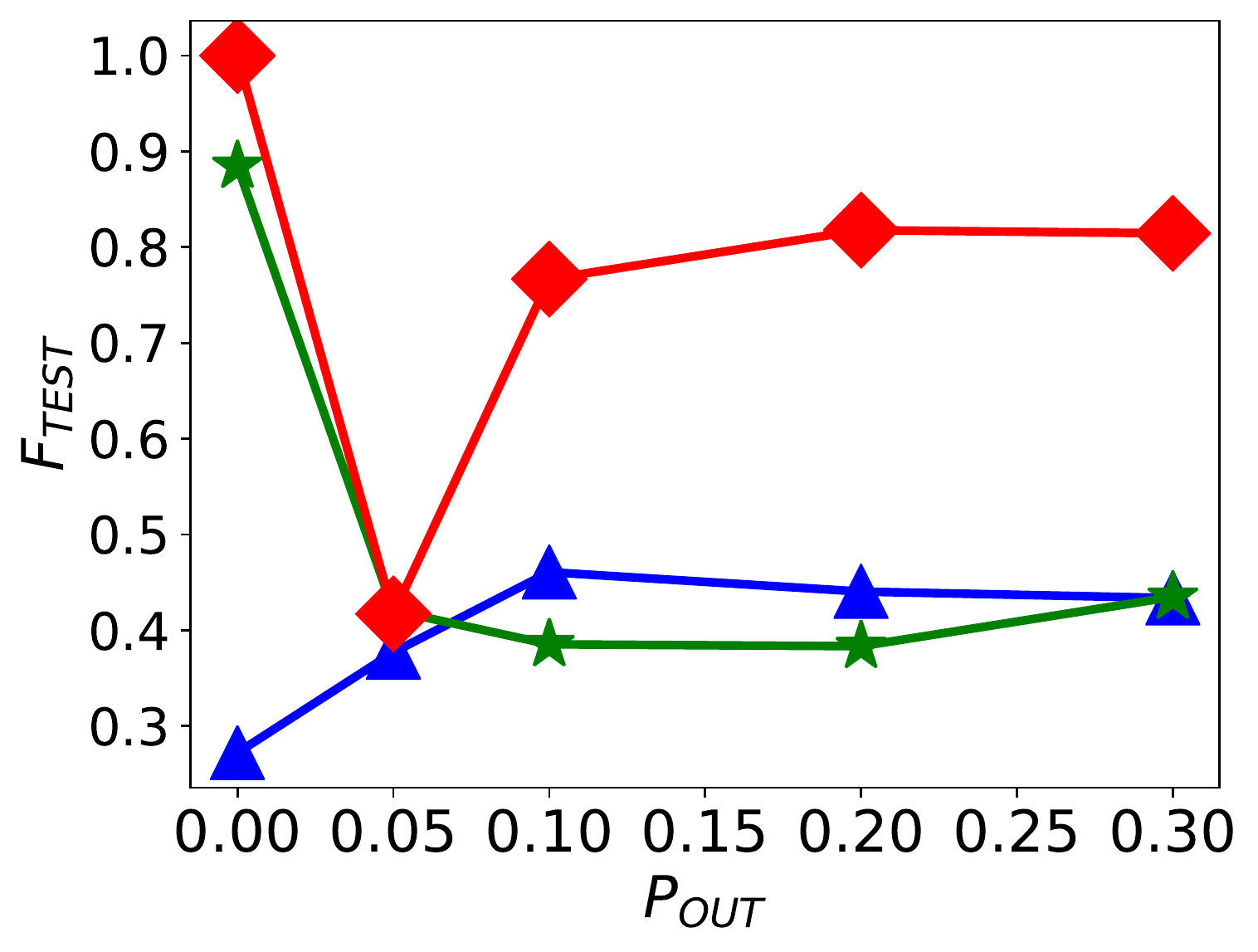}\label{fig:ftest_scale_2}}
 %   \vspace{-3mm}
  %  \subfloat[$p=1.5$]{\includegraphics[width=0.5\textwidth]{plots/linePlot_MNIST_p1.5_TEST.pdf}\label{fig:ftest_scale_1.5}}
    \subfloat[Test Loss, $p=1$]{\includegraphics[width=0.33\textwidth]{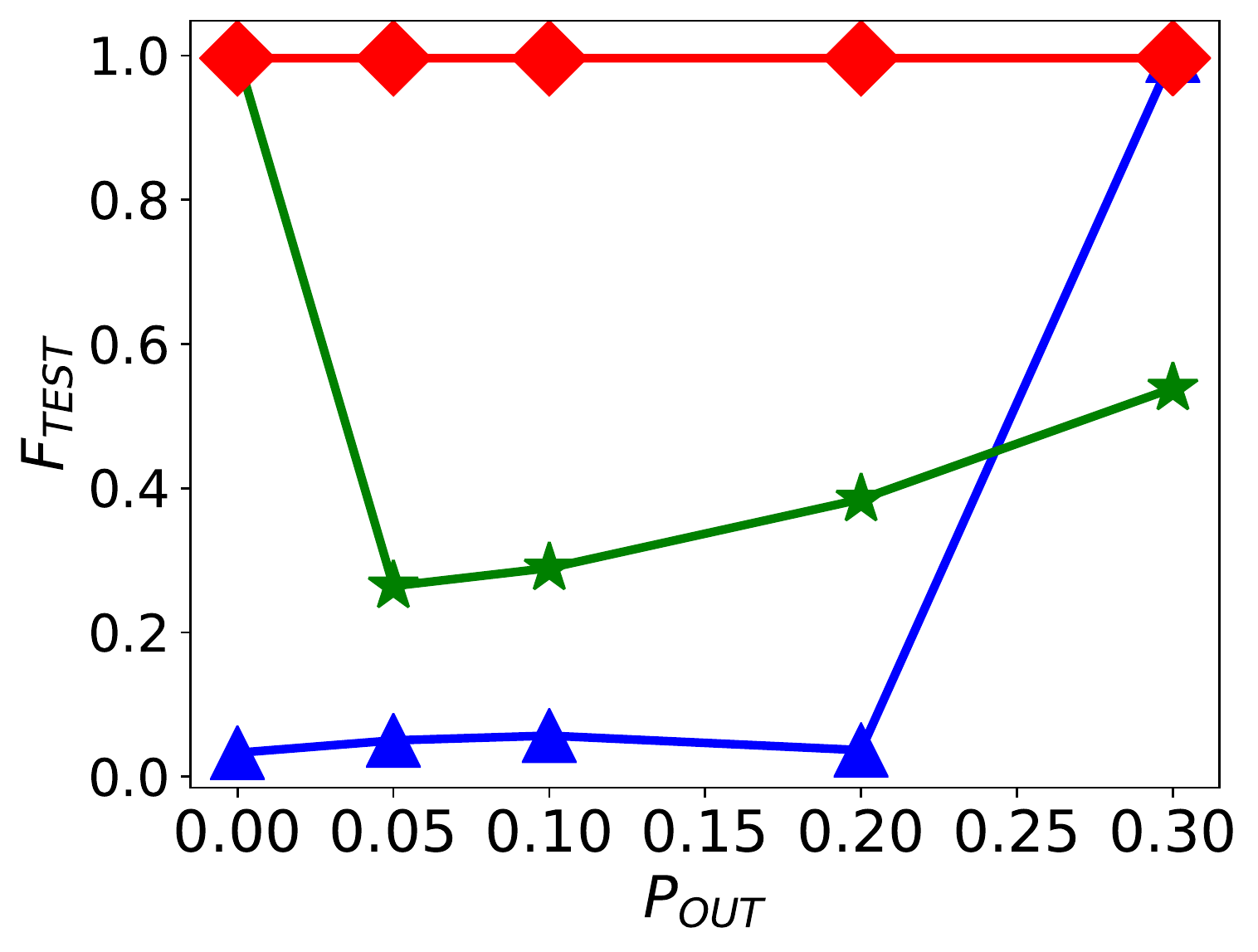}\label{fig:ftest_scale_1}}
    \caption{A   comparison of scalability  of the non-outlioers loss $F_{\texttt{NOUT}}$ and the test loss $F_{\texttt{TEST}}$  for different $p$-norms, w.r.t.,  outliers fraction $\outl.$ We normalize values in each figure by the largest observed value, to make comparisons between different objectives possible. We see that MSE in Figures~\ref{fig:fout_scale_sq} and
     \ref{fig:ftest_nout_scale} are drastically affected by outliers and scale with outliers fraction $\outl$. Other $\ell_p$ norms for different methods in Figures~\ref{fig:fout_scale_2}, \ref{fig:fout_scale_1}, \ref{fig:ftest_scale_2}, and \ref{fig:ftest_scale_1} generally stay unchanged w.r.t. $\outl.$ However, \mboa{} in the high outlier regime and $p=1$ performs poorly.}\label{fig:ftest_nout_scale}
\end{figure}

\subsection{Robustness Analysis}\label{sec:robust} We further  study   the  robustness of different $p$-norms  and MSE to the  presence of outliers. For brevity, we only report results for \texttt{MNIST} and for $p=1,2,$ and MSE. For more results refer to Fig.~\ref{fig:ftest_nout_scale_app} in \fullversion{App.~\ref{app:robust} in  \cite{our}.}{App.~\ref{app:robust}.}  We show the scaling of $F_{\texttt{NOUT}}$  and $F_{\texttt{TEST}}$ w.r.t. the   fraction $\outl$ in Fig.~\ref{fig:ftest_nout_scale}, for different  norms. To make comparisons between different objectives interpretable, we normalize all values in each figure by the largest value in that figure. 

By comparing Figures~\ref{fig:fout_scale_sq} and \ref{fig:ftest_scale_ell2}, corresponding to MSE,  with other plots in Fig.~\ref{fig:ftest_nout_scale}, we see that the loss values considerably increase by adding outliers. For other $p$-norms, we see that \sgd{} generally stays unchanged, w.r.t. outliers. However, the loss for \sgd{} is higher than MBO variants.  Loss values for \mbob{} also do not increase significantly by adding outliers. Moreover, we see that,  when no outliers are present $\outl=0.0$, \mbob{} obtains higher loss values. \mboa{} generally achieves  the lowest loss values and these values again do not increase with increasing $\outl$; however, for  the highest outliers ($\outl=0.3$), the performance of \mboa{} is considerably worse for $p=1.$ As we emphasize in Sec.~\ref{sec:time_obj_pfm}, high number of outliers  adversely affects the convergence of SADM, and hence the poor performance of \mboa{} for $\outl=0.3.$

\begin{figure}[t]
    \subfloat[\texttt{MNIST}, MSE]{\includegraphics[width=0.33\textwidth]{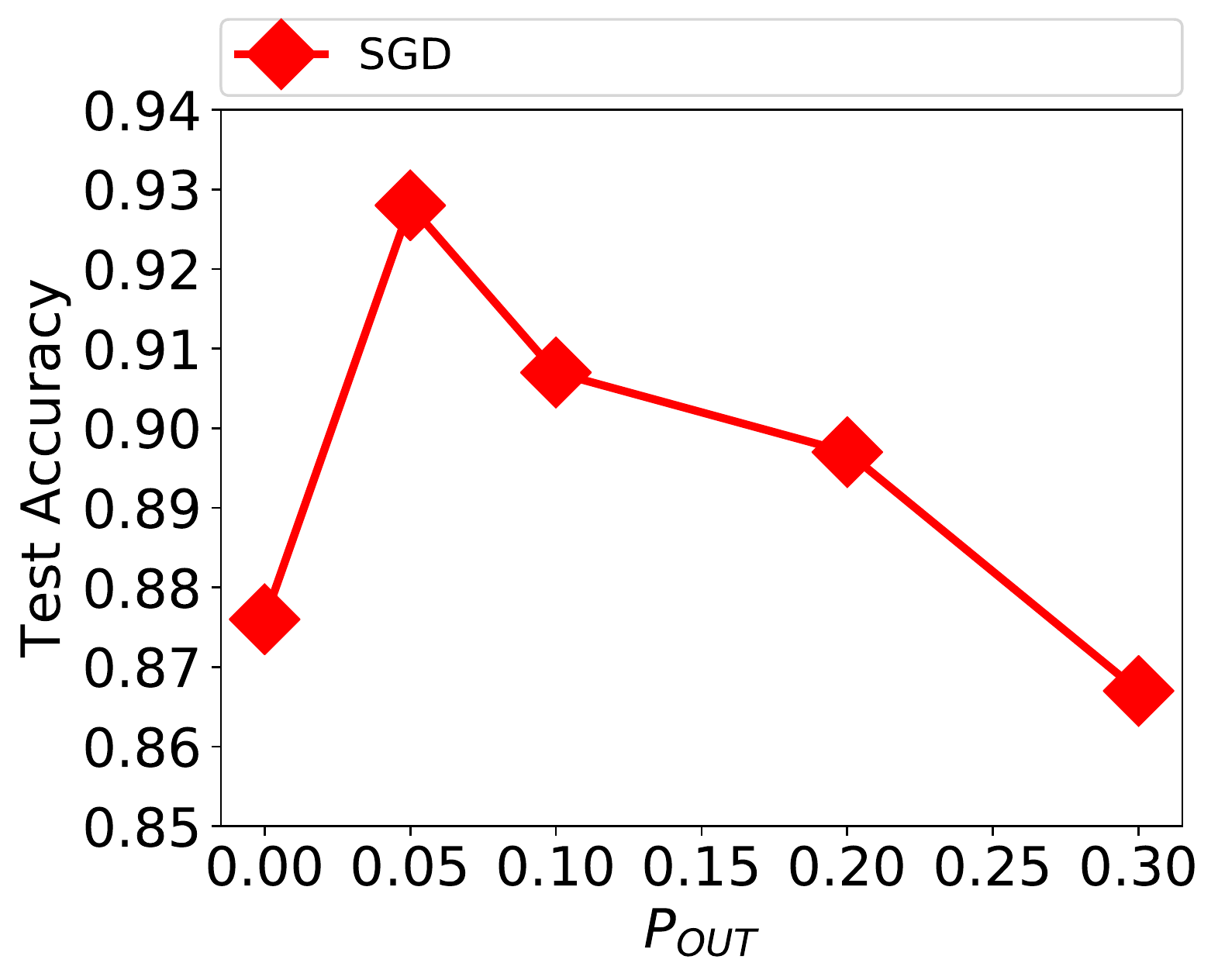}\label{fig:acc_mnist_mse}}
    \subfloat[\texttt{MNIST}, $p=2$]{\includegraphics[width=0.33\textwidth]{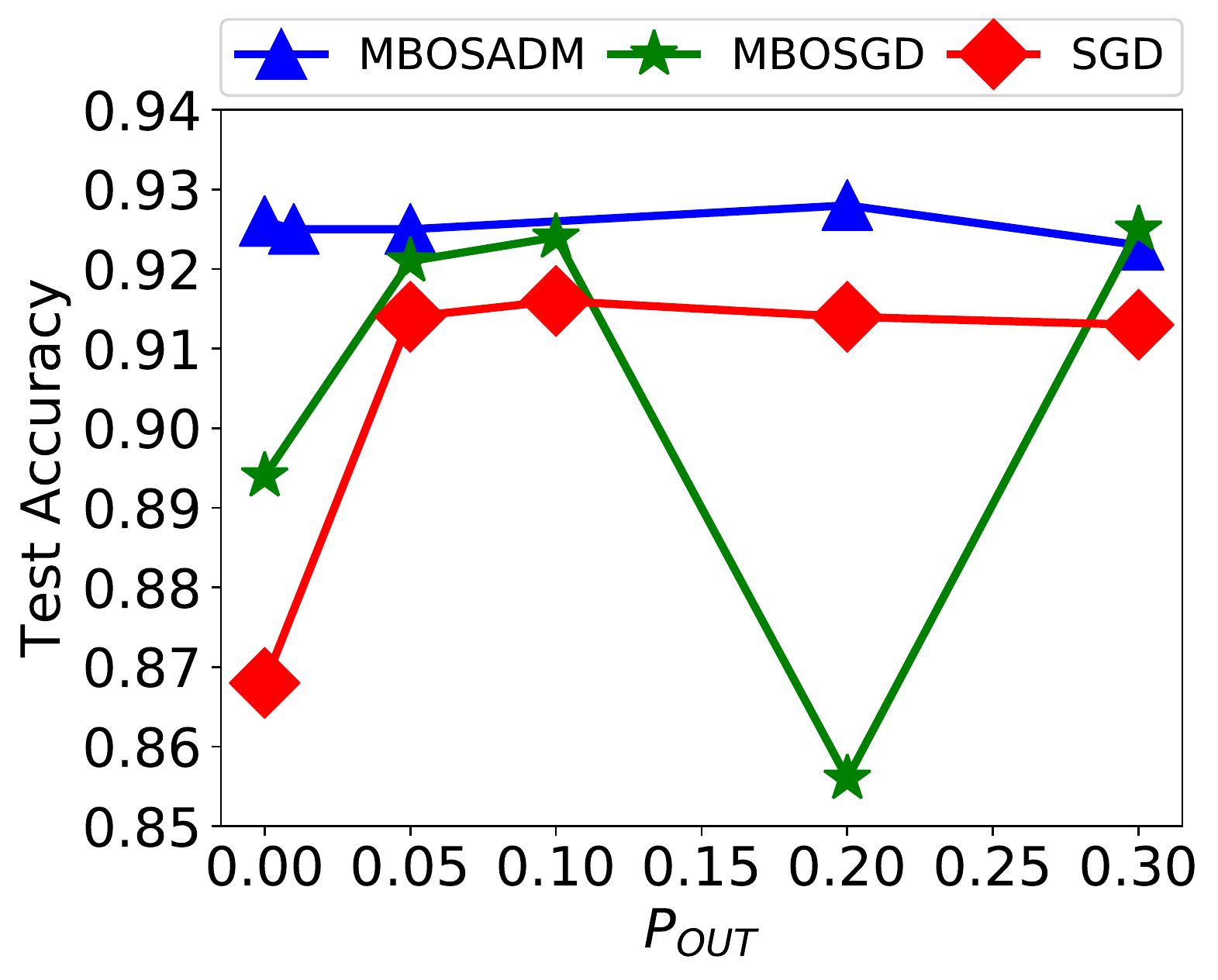}\label{fig:acc_mnist_p2}}
    \subfloat[\texttt{MNIST}, $p=1$]{\includegraphics[width=0.33\textwidth]{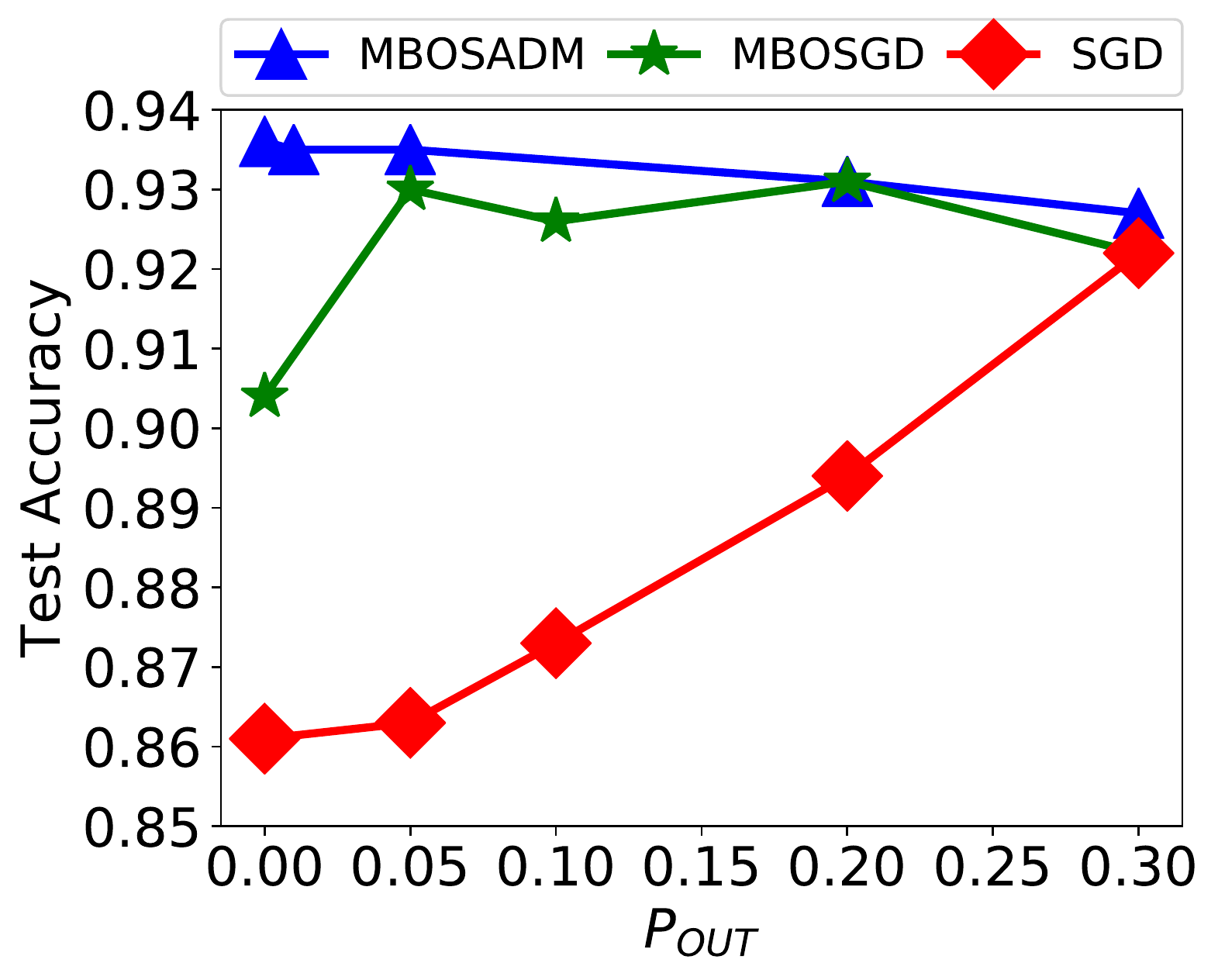}\label{fig:acc_mnist_p1}}
    
    \subfloat[\texttt{Fashion-MNIST}, MSE]{\includegraphics[width=0.33\textwidth]{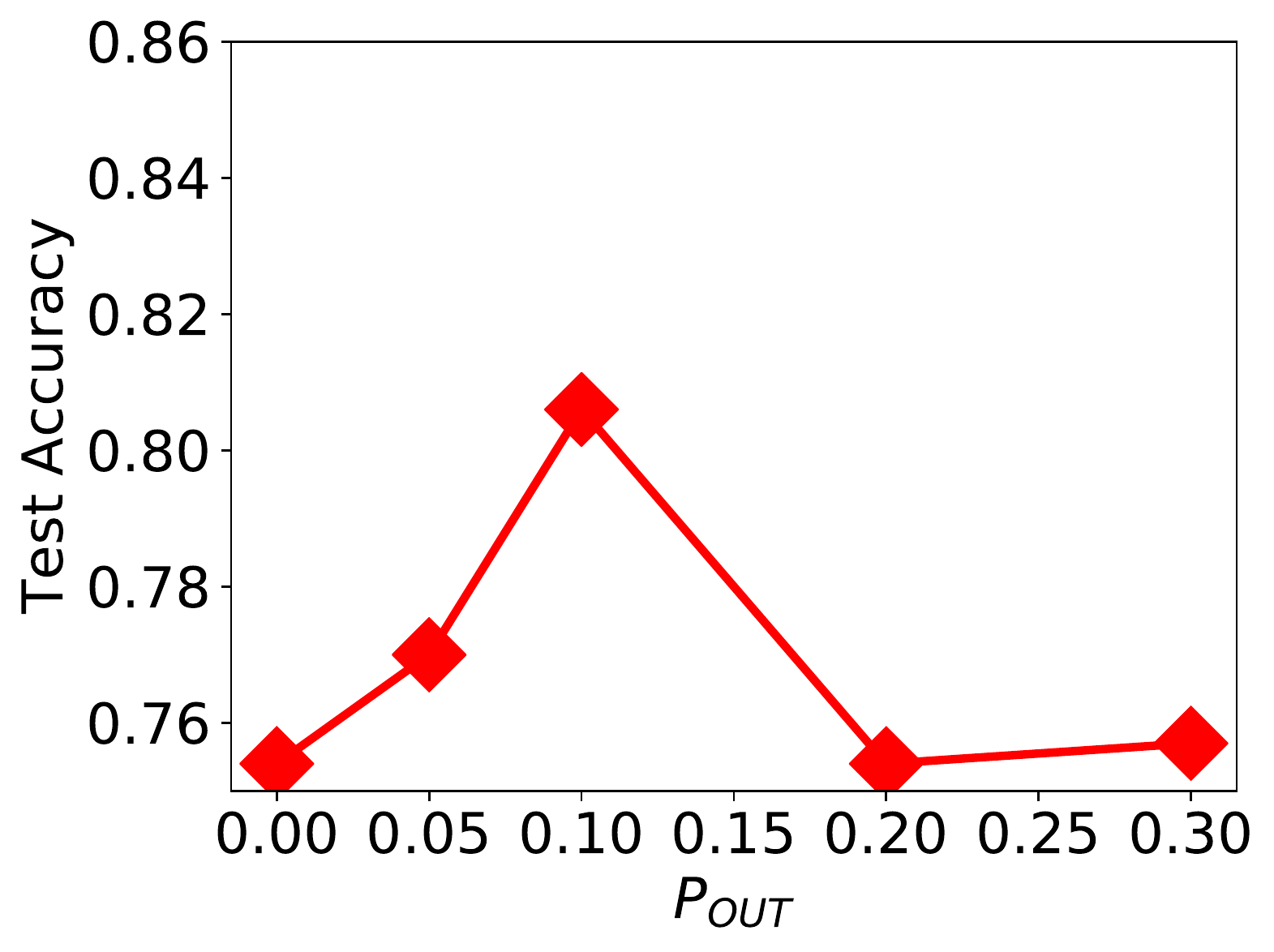}\label{fig:acc_fmnist_mse}}
    \subfloat[\texttt{Fashion-MNIST}, $p=2$]{\includegraphics[width=0.33\textwidth]{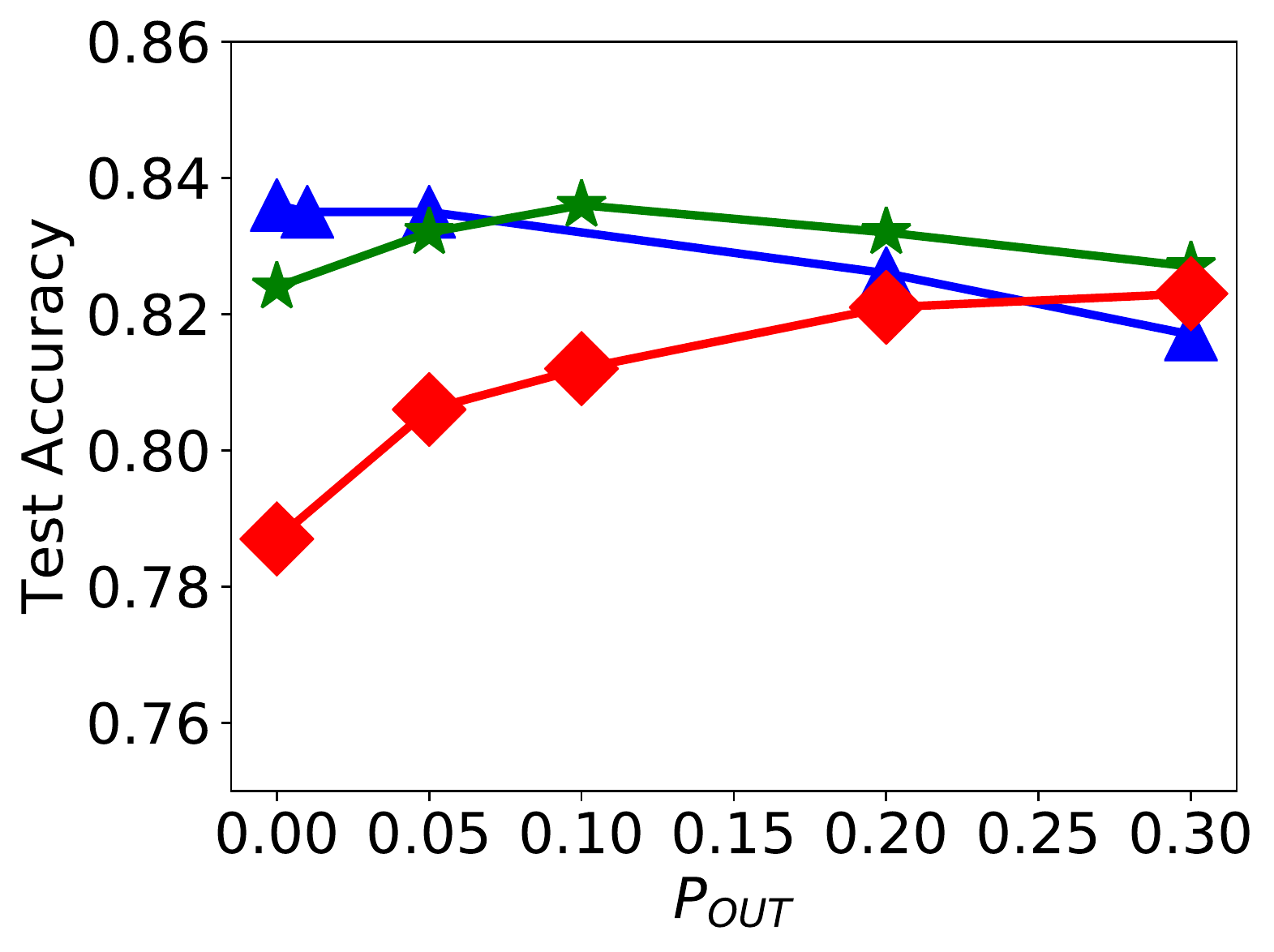}\label{fig:acc_fmnist_p2}}
    \subfloat[\texttt{Fashion-MNIST}, $p=1$]{\includegraphics[width=0.33\textwidth]{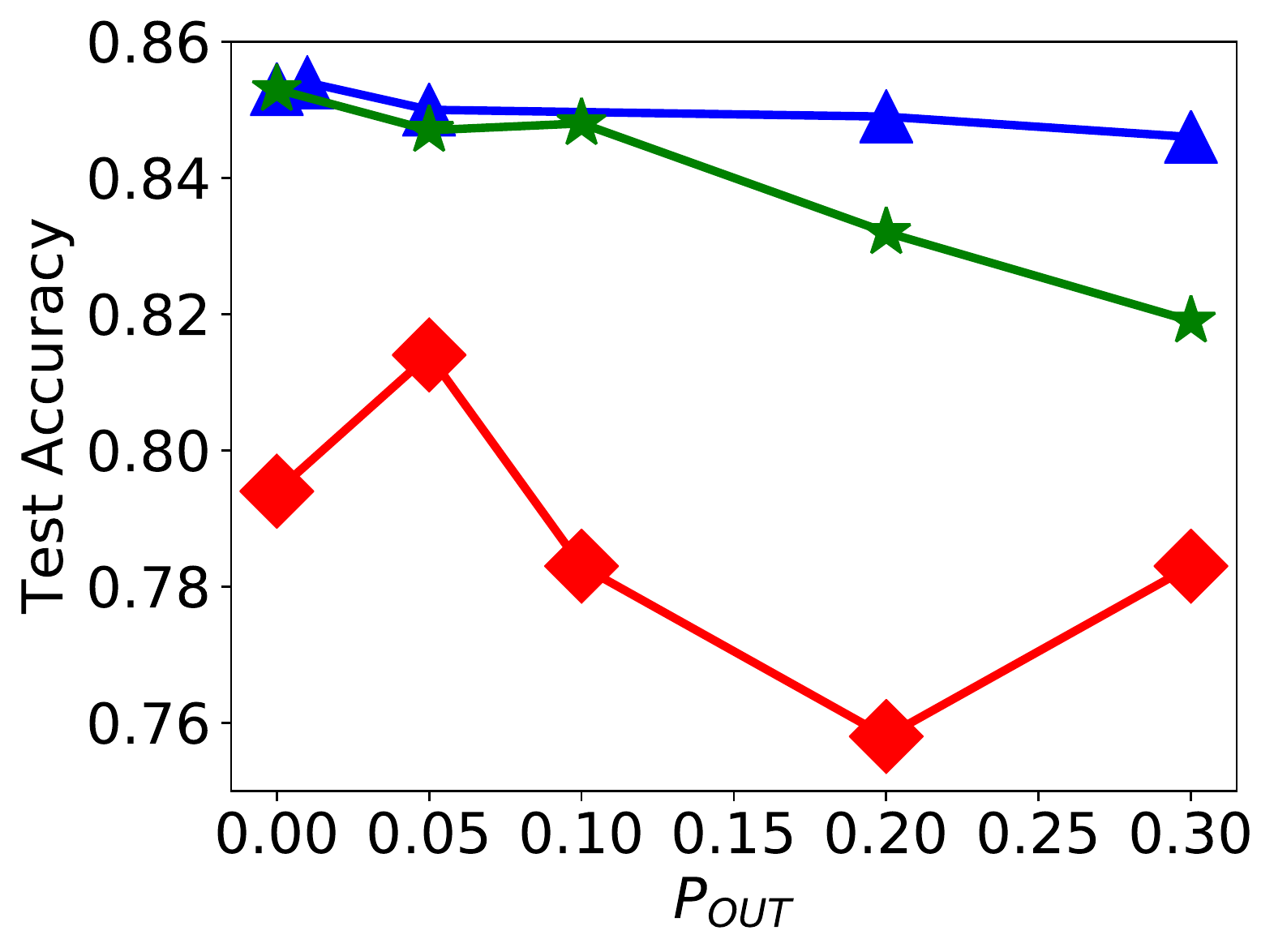}\label{fig:acc_fmnist_p1}}
    \caption{Classification performance for different methods and datasets. We use the embeddings obtained by auto-encoders trained via different algorithms to train a logistic regression model for classification. We generally observe that \mboa{} results in higher accuracy on the test sets. Moreover, we see that MSE is evidently sensitive to outliers, see Figures~\ref{fig:acc_mnist_mse} and \ref{fig:acc_fmnist_mse} for $\outl\geq 0.2.$} \label{fig:cls_task}
\end{figure}

\subsection{Classification Performance}\label{sec:exp_cls} Fig.~\ref{fig:cls_task} shows the quality of the latent embeddings obtained by different trained autoencoders on the downstream classification over  \texttt{MNIST} and \texttt{FashionMNIST}. Additional results are shown in Tables~\ref{tab:acc_mnist} and \ref{tab:acc_fmnist} in App.~\ref{app:classification}. We see that MBO variants again outperform \sgd{}. For \texttt{MNIST}, (reported in Figures~\ref{fig:acc_mnist_mse} to \ref{fig:acc_mnist_p1}), we see that \mboa{} for $p=1$ obtains the highest accuracy. Moreover, for \texttt{Fashion-MNIST} (reported in Fig.~\ref{fig:acc_fmnist_mse} to \ref{fig:acc_fmnist_p1}), we observe that again \mboa{} for $p=1$ outperforms other methods.  
We also observe that MSE (reported in Figures~\ref{fig:acc_mnist_mse} and \ref{fig:acc_fmnist_mse}) is sensitive to outliers; the corresponding accuracy drastically drops for $\outl\geq 0.1.$ An interesting  observation is that adding outliers improves the performance of \sgd{}; however, we see that \sgd{} always results in lower accuracy, except in two cases ($\outl=0.2$ in Fig.~\ref{fig:acc_mnist_p2} and $\outl=0.3$ in Fig.~\ref{fig:acc_fmnist_p2}). 

%\mboa{} achieves the highest accuracy for outliers  $\outl=0.0, 0.05, 0.1, 0.2$ for $p=1.0$ and for $\outl=0.3$ \mbob{} obtains the highest accuracy for $p=1.5.$ We see a similar result in Table~\ref{tab:acc_fmnist}, where \mboa{} for $p=1$ achieves the highest accuracy in all outlier settings. We also observe that  for the $\ell_2$ squared the classification accuracy slightly increases by adding small number of  outliers (e.g., compare values for $\outl=0.0, 0.05, 0.1$); however, for more number of outliers the accuracy deteriorates. Moreover, we see that \mboa{} and \mbob{} (for  other $\ell_p$ norms) always show  better performance than the case of $\ell_2$ squared. 

\section{Conclusion}\label{sec:conclusion}
We present a generic class of robust formulations that includes many  applications, i.e., auto-encoders, multi-target regression, and  matrix factorization. We show that  SADM, in combination with MBO, provides efficient solutions for our class of robust problems. 
 Studying  other proximal measures described by Ochs et al.~\cite{Ochs2019} is an open area. Moreover, characterizing  the sample complexity of our proposed method for obtaining a stationary point, as in  MBO variants that use gradient methods \cite{drusvyatskiy2019efficiency,davis2019stochastic}, is an interesting future direction.

%\section*{Acknowledgements}

%The authors gratefully acknowledge support from the National Science Foundation (Grants  CCF-1750539, IIS-1741197, and  CNS-1717213).

\bibliographystyle{splncs04}
\bibliography{main}

\newpage
%\begin{adjustwidth}{-22mm}{-22mm}

\appendix

\section{Model Based Optimization}\label{app:lsa}

Ochs et al.~\cite{Ochs2019} allows to use more general Bregman divergences for the second term (c.f. Sec.~5.2 of \cite{Ochs2019}). The specific model functions we study fall under Example 5.3 in \cite{Ochs2019}  (see also \cite{lewis2016proximal,davis2019stochastic}.) %Note that in case, where the function $F$ is differentiable, the model function can be  the first-order Taylor approximation around $\vc{\theta}^k$. Then the algorithm reduces to the projected gradient descent algorithm.  
Moreover, Ochs et al. allow \eqref{eq:unconstainedsub} to be solved inexactly; at each iteration, the solution $\Tilde{\vc{\theta}}^k$ only needs to improve the model function value by, i.e., 
\begin{align}\label{eq:model_improve}
 \Delta_k\triangleq F_{\vc{\theta}^k}(\Tilde{\vc{\theta}}^k) + \frac{h}{2}\|\Tilde{\vc{\theta}}^k- \vc{\theta}^k\|_2^2 - F_{\vc{\theta}^k}(\vc{\theta}^k) <0.  
  \end{align}
The step-size $\eta^k$ is found via an Armijo line search algorithm. In particular, the line search algorithm finds a step-size, s.t., $\vc{\theta}^{k+1}$  improves the current objective comparable with the model improvement $\Delta^k$, i.e., 
$$\frac{1}{n} \sum_{i\in[n]} \|F(\vc{\theta}^{k+1}; \vc{x}_i)\|_p+ g(\vc{\theta}^{k+1})- \left(\frac{1}{n} \sum_{i\in[n]} \|F(\vc{\theta}^k; \vc{x}_i)\|_p + g(\vc{\theta}^k) \right) \leq  \gamma \eta^k \Delta^{k},$$
where $\gamma \in (0,1)$ is a hyper-parameter of the linear search algorithm.

The exact Line Search Algorithm from Ochs et al.~\cite{Ochs2019} is summarized in Alg.~\ref{alg:lsa}. Note that Ochs et al. prove that  LSA is guaranteed to finish within finite number of iterations. 

\begin{algorithm}[t!] 
    \caption{Line Search Algorithm (LSA)}
    \label{alg:lsa}
    \begin{algorithmic}[1] % The number tells where the line numbering should start
        \State \textbf{Input:} Solutions $\vc{\theta}^k$,  $\Tilde{\vc{\theta}}^k$, and parameters  $\delta, \gamma \in (0,1), \tilde{\eta} > 0$
        \State Initialize  $\Tilde{\vc{\theta}}:=(1-\tilde{\eta}) \vc{\theta}^k + \tilde{\eta} \Tilde{\vc{\theta}}^k$
        \While{\small{$\frac{1}{n} \sum_{i\in[n]} \|F(\Tilde{\vc{\theta}}; \vc{x}_i)\|_p+ g(\Tilde{\vc{\theta}})  > \frac{1}{n} \sum_{i\in[n]} \|F(\vc{\theta}^k; \vc{x}_i)\|_p+ g(\vc{\theta}^k) + \gamma \tilde{\eta} \big [ F_{\vc{\theta}^k}(\Tilde{\vc{\theta}}^k) - F_{\vc{\theta}^k}(\vc{\theta}^k)+ \frac{h}{2}\|\Tilde{\vc{\theta}}^k- \vc{\theta}^k\|_2^2 \big ]$}}
        \State Set  $\tilde{\eta} := \delta\tilde{\eta} $ 
        \State Set $\Tilde{\vc{\theta}}:=(1-\tilde{\eta}) \vc{\theta}^k + \tilde{\eta} \Tilde{\vc{\theta}}^k$
        \EndWhile
        	\State \textbf{Return:}  $\tilde{\eta}$
        \end{algorithmic}
\end{algorithm}

%\icmltitle{$\ell_p$ Norm Loss for Robust Regression via Model Based Methods}
\section{Proof of Proposition~\ref{prop:ochs}} \label{append:proofpropochs}
We show that all assumptions for Theorem~4.1 of Ochs et al.~\cite{Ochs2019} are satisfied, therefore the result holds. 
We solve convex problems \eqref{eq:unconstainedsub} via OADM iterations. Since we established the $O(\frac{\log T}{T})$ convergence rate of OADM in Theorem~\ref{trm:oadm}, we can solve \eqref{eq:unconstainedsub} with an arbitrary accuracy $\epsilon$, which goes to zero for $T\to \infty$; therefore, Assumption~4.1 of \cite{Ochs2019} is  satisfied. 
Moreover, due to the choice of the Euclidean norm as our Bergman
distance function, Assumption~4.2 of \cite{Ochs2019} is satisfied (see Section 5 of \cite{Ochs2019}). Finally, our model function \eqref{eq:MF} is precisely Example~5.3 in \cite{Ochs2019}; it is written as the form $f_0+h\circ F,$ where $f_0(\vc{\theta})=\chi_{\mathcal{C}}(\vc{\theta}) + g(\vc{\theta})$, $h(\vc{F})=\|\vc{F}\|_{p,1}$, and $F:\reals^d\to\reals^{n\times N}.$ Therefore, Assumption~4.3 is also satisfied. In addition, the domain of Euclidean distance is $\mathbb{R}^n$, which implies that Condition (ii) in Theorem 4.1 of \cite{Ochs2019} is satisfied for any limit point. Thus, all the conditions in Theorem~4.1 of Ochs et al. are satisfied, and every limit point of Alg.~\ref{alg:mbo} is a stationary point.

\section{Proof of Theorem~\ref{trm:oadm}} \label{append:proofoadm}
\begin{proof}
Since we assume that the constraint set $\mathcal{C}$ is convex, closed, and bounded, we take the diameter of the set $\mathcal{C}$ to be 
\begin{align}\label{eq:diam}
D_{\mathcal{C}}=\max_{\theta_1, \theta_2\in \mathcal{C}} \|\theta_1 - \theta_2\|_2.
\end{align}
In addition, since function $F$ and its Jacobian $\vc{D}_{F_i}(\vc{\theta})$ are bounded on the set $\mathcal{X}$, for all $i\in[n]$, there exist $M_F, M_D<\infty$, s.t.,  
\begin{subequations}\label{eq:bounded}
\begin{align}
  \|F(\vc{\theta}; \vc{x})\|_{\infty}&\leq M_F~~\forall \vc{\theta}\in \mathcal{C}, \vc{x}\in\{\vc{x}_1,\ldots,\vc{x}_n\}\label{eq:bounded_F} \\
  \|\vc{D}_{F_i}(\vc{\theta})\|_{\infty} &\leq M_D ~~\forall \vc{\theta}\in \mathcal{C}, i\in [n].\label{eq:bounded_D}  
\end{align}
\end{subequations}

We now paraphrase Theorem~6 of \cite{wang2013online} as the following lemma, which we prove in Appendix~\ref{append:proof_oadm_lemma}. 
\begin{lemma}[Theorem 6 in \cite{wang2013online}]\label{lem:oadm}
If the assumptions of Theorem~\ref{trm:oadm} hold,  for the sequence $\{\vc{\theta}^t_1, \vc{\theta}^t_2, \vc{u}^t\}, t\in[T]$ generated by OADM algorithm for any sequence of the variables $\vc{x}_t\in \{\vc{x}_1,\ldots,\vc{x}_n\}, t \in [T]$ the  following holds
\begin{subequations}\label{eq:bounds_oadm}
\begin{align}\label{eq:obj_bound}
     & \sum_{t=1}^T  \left(F^{(k)}(\vc{\theta}_1^t; \vc{x}_t) + G(\vc{\theta}_2^{t+1})\right)  -   \sum_{t=1}^T   \left(  F^{(k)}(\vc{\theta}^*; \vc{x}_t )+ G(\vc{\theta}^*)
    \right) \nonumber \\
    &\leq \frac{(N^{1/p}d M_{D} + hD_{\mathcal{C}})^2}{2h}\log(T+1) + \frac{\beta+h}{2} D_{\mathcal{C}}^2 \\
    &\sum_{t=1}^T \|\vc{\theta}_1^{t+1} - \vc{\theta}_2^{t+1}\|_2^2 + \|\vc{\theta}_2^{t+1} - \vc{\theta}_2^t\|_2^2 \nonumber \\
    &\leq \frac{2}{\beta} (\sqrt{d} M_G + h D_{\mathcal{C}} + L_G) \log(T+1) + (1+ \frac{h}{\beta})D_{\mathcal{C}}^2,
    \label{eq:feas_bound}
\end{align} 
\end{subequations}
where $\vc{\theta}^* = \vc{\theta}_1^*= \vc{\theta}_2^*$ is the optimal solution for \eqref{eq:admm_form}, and $L_G$ is the Lipschitz coefficient for $g(\cdot)$, and
$M_G\triangleq N^{1/p}\sqrt{d} M_{D}$.
\end{lemma}
From \eqref{eq:feas_bound} we obtain the following 
\begin{subequations}\label{eq:cor:eq2}
\begin{align}
\sum_{t=1}^T \|\vc{\theta}_1^{t} - \vc{\theta}_2^{t}\|_2^2 = O(\log T)\\
\sum_{t=1}^T \|\vc{\theta}_2^{t+1} - \vc{\theta}_2^{t}\|_2^2=  O(\log T).
\end{align}
\end{subequations}
Now we derive the result
\begin{align*}
    \|\bar{\vc{\theta}}^T_1 - \bar{\vc{\theta}}^T_2 \|^2_2&= \|\frac{1}{T} \sum_{t=1}^T (\vc{\theta}_1^t - \vc{\theta}_2^{t+1})\|^2_2 \\
  & = \|\frac{1}{T}\sum_{t=1}^T (\vc{\theta}_1^t - \vc{\theta}_2^{t} + \vc{\theta}_2^{t}  - \vc{\theta}_2^{t+1} )\|^2_2 \\
  &\leq \|\frac{2}{T}\sum_{t=1}^T (\vc{\theta}_1^t - \vc{\theta}_2^{t} )\|_2^2 + \|\frac{2}{T}\sum_{t=1}^T (\vc{\theta}_2^{t}  - \vc{\theta}_2^{t+1} )\|^2_2\\
    &{\leq} \frac{2}{T} \sum_{t=1}^T \|\vc{\theta}_1^t - \vc{\theta}_2^t\|^2_2 + \frac{2}{T} \sum_{t=1}^T \|\vc{\theta}_2^{t}  - \vc{\theta}_2^{t+1}\|^2_2\\
    &\stackrel{\mbox{\eqref{eq:cor:eq2}}}{=} O(\frac{\log T}{T}),
\end{align*}
where in deriving the first inequality we have used the fact that $$\|\vc{x}+\vc{y}\|_2^2 \leq 2\|\vc{x}\|_2^2 +2\|\vc{y}\|_2^2~ \forall \vc{x},\vc{y}\in\reals^d.$$

Now we prove the second part of theorem about the optimality of solutions. 
Using convexity of $F^{(k)}$ and $G$ we have that 
\begin{align}\label{eq:proof}
    &F^{(k)}(\bar{\vc{\theta}}^T_1) + G(\bar{\vc{\theta}}^{T}_2)- F^{(k)}(\vc{\theta}^*)- G(\vc{\theta}^*) \nonumber \\ 
    &\leq \frac{1}{T}\sum_{t=1}^T\left(F^{(k)}(\vc{\theta}_1^t)  - F^{(k)}(\vc{\theta}^*)+G(\vc{\theta}_2^{t+1})  - G(\vc{\theta}^*)\right)\nonumber\\
    &= \frac{1}{T}\sum_{t=1}^T\left(F^{(k)}(\vc{\theta}_1^t; \vc{x}_t)  - F^{(k)}(\vc{\theta}^*; \vc{x}_t)+G(\vc{\theta}_2^{t+1})  - G(\vc{\theta}^*)\right)\nonumber\\ &+  \frac{1}{T}\sum_{t=1}^T\left(F^{(k)}(\vc{\theta}_1^t) - F^{(k)}(\vc{\theta}_1^t; \vc{x}_t) -F^{(k)}(\vc{\theta}^*) +  F^{(k)}(\vc{\theta}^*; \vc{x}_t)\right)\nonumber\\&\stackrel{{\eqref{eq:obj_bound}}}{\leq}
     \frac{(N^{1/p}d M_{D} + hD_{\mathcal{C}})^2}{2hT}\log(T+1) + \frac{\beta+h}{2T} D_{\mathcal{C}}^2 +\frac{1}{T}\sum_{t=1}^T \delta_t,
\end{align}
where the first inequality is due to the Jensen's inequality and $$\delta_t\triangleq F^{(k)}(\vc{\theta}_1^t) - F^{(k)}(\vc{\theta}_1^t; \vc{x}_t) -F^{(k)}(\vc{\theta}^*) +  F^{(k)}(\vc{\theta}^*; \vc{x}_t).$$ 
As the variables $\vc{x}_t$ and $\vc{\theta}_1^t$ are independent, we have 
\begin{align*}
    &\E[\delta_t|\vc{x}_1,\ldots\vc{x}_{t-1}] \\
    &= \E\left[F^{(k)}(\vc{\theta}_1^t) - F^{(k)}(\vc{\theta}_1^t; \vc{x}_t) -F^{(k)}(\vc{\theta}^*) +  F^{(k)}(\vc{\theta}^*; \vc{x}_t)|\vc{x}_1,\ldots\vc{x}_{t-1}\right] = 0.
\end{align*}
Therefore, we obtain
\begin{align}\label{eq:zeromean}
\E_{\vc{x}_t, t\in[T]} [\delta_t] = 0. 
\end{align}
Now taking expectations of both sides of \eqref{eq:proof} w.r.t. the sequence $\zeta_t, t\in[T]$ and noting \eqref{eq:zeromean} we have that:
\begin{align*}
    &\E_{\vc{x}_t, t\in [T]}\left[F^{(k)}(\bar{\vc{\theta}}^T_1) + G(\bar{\vc{\theta}}^{T}_2)- F^{(k)}(\vc{\theta}^*)- G(\vc{\theta}^*)\right] \\
    &\leq \frac{(N^{1/p}d M_{D} + hD_{\mathcal{C}})^2}{2hT}\log(T+1) + \frac{\beta+h}{2T} D_{\mathcal{C}}^2\\ 
    &=O(\frac{\log T}{T})
\end{align*}

We first derive the following for all $\vc{\theta}\in \mathcal{C}, \vc{x}_i\in\{\vc{x}_1,\ldots,\vc{x}_n\}$:
{\small
\begin{align}\label{eq:bound_fi}
    &|F^{(k)}(\vc{\theta}) - F^{(k)}(\vc{\theta}; \vc{x}_i)| \nonumber \\
    &= \left|F_{\vc{\theta}^{(k)}}(\vc{\theta}; \vc{x}_i) - \frac{1}{n}\sum_{j=1}^n F_{\vc{\theta}^{(k)}}(\vc{\theta}; \vc{x}_j)\right| \nonumber\\ 
    &\leq \frac{1}{n}\sum_{j\neq i}\left|F_{\vc{\theta}^{(k)}}(\vc{\theta}; \vc{x}_i)  -F_{\vc{\theta}^{(k)}}(\vc{\theta}; \vc{x}_j) \right|\nonumber\\
    &= \frac{1}{n}\sum_{j\neq i} \bigm| \|F(\vc{\theta}^{(k)}; \vc{x}_i)+ \vc{D} _{F_i}(\vc{\theta}^{(k)})(\vc{\theta}-\vc{\theta}^{(k)}))\|_p\nonumber \\
    &- \|F(\vc{\theta}^{(k)}; \vc{x}_j)+ \vc{D} _{F_j}(\vc{\theta}^{(k)})(\vc{\theta}-\vc{\theta}^{(k)})\|_p\bigm| \nonumber\\
    &\leq \frac{1}{n}\sum_{j\neq i}\|F(\vc{\theta}^{(k)}; \vc{x}_i)+ \vc{D} _{F_i}(\vc{\theta}^{(k)})(\vc{\theta}-\vc{\theta}^{(k)})) -
    F(\vc{\theta}^{(k)}; \vc{x}_j)+ \vc{D} _{F_j}(\vc{\theta}^{(k)})(\vc{\theta}-\vc{\theta}^{(k)})\|_p \nonumber\\
    &\leq \frac{1}{n}\sum_{j\neq i} \|F(\vc{\theta}^{(k)}; \vc{x}_i) - F(\vc{\theta}^{(k)};\vc{x}_j)\|_p + 
    \frac{1}{n}\sum_{j\neq i}\left\|\left(\vc{D}_{F_i}(\vc{\theta}^{(k)})- \vc{D} _{F_j}(\vc{\theta}^{(k)})\right)(\vc{\theta}-\vc{\theta}^{(k)})\right\|_p \nonumber\\ &\stackrel{\eqref{eq:bounded_F}}{\leq} 
    \sqrt[p]{N}M_F +\frac{1}{n}\sum_{j\neq i}\left(\sum_{i'=1}^N\bigm | \left(\vc{D}_{F_i}(\vc{\theta}^{(k)})- \vc{D} _{F_j}(\vc{\theta}^{(k)})\right)_{i'}^\top(\vc{\theta}-\vc{\theta}^{(k)})\bigm| \mid^p\right)^{1/p}\nonumber\\ &\leq \sqrt[p]{N}M_F+ \frac{1}{n}\sum_{j\neq i}\left(\sum_{i'=1}^N\left(\left\|\left(\vc{D}_{F_i}(\vc{\theta}^{(k)})- \vc{D}_{F_j}(\vc{\theta}^{(k)})\right)_{i'}\right\|_2\|\vc{\theta}-\vc{\theta}^{(k)}\|_2\right)^p\right)^{1/p}\nonumber  \\ &\stackrel{\eqref{eq:bounded_D}, \eqref{eq:diam}}{\leq}
    \sqrt[p]{N}M_F+\sqrt[p]{N}\sqrt{d}M_DD_{\mathcal{C}}.
\end{align}
}
%where (i) is due to the assumption that function $F(\vc{\theta}, \vc{x})$ is bounded on the set $\mathcal{C}$, for all $\vc{x}\in\{\vc{x}_1,\ldots, \vc{x}_n\}$, and (ii) is due to the assumption that Jacobian $\vc{D}_{F_i}(\vc{\theta})$ is bounded on the set $\mathcal{X}$, for all $i\in[n]$, and the fact that $D_{\mathcal{C}}=\max_{\theta_1, \theta_2\in \mathcal{C}} \|\theta_1 - \theta_2\|_2.$

In order to show the rest of the results we first show that the variance of $\delta_t$ are bounded for all $t\in [T]$
\begin{align}
\delta_t^2\nonumber 
&= 
\left(F^{(k)}(\vc{\theta}_1^t) - F^{(k)}(\vc{\theta}_1^t, \zeta_t) -F^{(k)}(\vc{\theta}^*) +  F^{(k)}(\vc{\theta}^*; \vc{x}_t)\right)^2\\\nonumber
&\leq 2\left(F^{(k)}(\vc{\theta}_1^t) - F^{(k)}(\vc{\theta}_1^t; \vc{x}_t)\right)^2 + 2\left( -F^{(k)}(\vc{\theta}^*) +  F^{(k)}(\vc{\theta}^*; \vc{x}_t)\right)^2\nonumber \\&\stackrel{\eqref{eq:bound_fi}}{\leq}
 4 \left(\sqrt[p]{N}M_F+\sqrt[p]{N}\sqrt{d}M_DD_{\mathcal{X}}\right)^2 \nonumber\\
 &\equiv \sigma^2.
\label{eq:squared}
\end{align}
From \eqref{eq:squared} it is obvious that  
\begin{align}\label{eq:exp}
\left(\frac{\delta_t^2}{\sigma^2}\right) \leq \exp(1)~~ \forall t\in [T].
 \end{align}
Now we show that the following holds
\begin{align}\label{eq:exp_gamma}
   \E[\exp(\alpha \delta_t)|\vc{x}_1, \ldots, \vc{x}_{t-1} ]\leq \exp(\alpha^2 \sigma^2)~~\forall\alpha>0.  
\end{align}
To show \eqref{eq:exp_gamma}, we follow  results from \cite{nemirovski2009robust}; similar to them, %we consider two cases:
%\begin{enumerate}
%    \item $0<\alpha \sigma\leq 1:$ Then 
   we use the fact that $\exp(x) \leq x+ \exp(x^2).$ Then we have that 
    \begin{align*}
         \E[\exp(\alpha \delta_t)|\vc{x}_1, \ldots, \vc{x}_{t-1} ]& \leq \E[\exp(\alpha^2 \delta_t^2)|\vc{x}_1, \ldots, \vc{x}_{t-1} ]\\&=
         \E\left[\left(\exp( \frac{\delta_t^2}{\sigma^2})\right)^{\alpha^2 \sigma^2} \bigm| \vc{x}_1, \ldots, \vc{x}_{t-1} \right]\\ &\stackrel{\eqref{eq:exp}}{\leq} \exp(\alpha^2 \sigma^2)
    \end{align*}
 %   \item $\alpha \sigma\geq 1:$ 
 %   \begin{align*}
 %       \E[\exp(\alpha \delta_t)|\zeta_1, \ldots, \zeta_{t-1} ] \leq \E\left[ \exp(\frac{1}{2} \alpha^2 \sigma^2 + \frac{1}{2} \frac{\delta_t^2}{\sigma^2})|\zeta_1, \ldots, \zeta_{t-1} \right] &\leq \\
 %       \E\left[ \exp(\frac{1}{2} \alpha^2 \sigma^2 + \frac{1}{2})|\zeta_1, \ldots, \zeta_{t-1} \right] \stackrel{\eqref{eq:squared}}{\leq}  \exp(\alpha^2 \sigma^2). 
 %   \end{align*}
%\end{enumerate}
Now for the sum $\sum_{t=1}^T\delta_t$ we have that
\begin{align}\label{eq:sum_exp}
 \E\left[\exp(\alpha \sum_{t=1}^T\delta_t)\right]& = \E\left[\exp(\alpha \sum_{t=1}^{T-1}\delta_t)\exp(\alpha\delta_T)\right] \nonumber\\ &=
 \E_{\vc{x}_1,\ldots,\vc{x}_{T-1}}\left[\exp(\alpha \sum_{t=1}^{T-1}\delta_t) \E_{\vc{x}_{T}}[\alpha\delta_T|\vc{x}_1,\ldots, \vc{x}_{T-1}]\right]\nonumber\\ & \stackrel{\eqref{eq:exp_gamma}}{\leq}
  \exp(\alpha^2 \sigma^2)\E\left[\exp(\alpha \sum_{t=1}^{T-1}\delta_t)\right].
 \end{align}
 Having \eqref{eq:sum_exp} for all $T$ and $\E[\exp(\alpha\delta_1)]\leq \exp(\alpha^2 \sigma^2)$ by induction  we obtain that:
 \begin{align}\label{eq:all_sum}
 \E\left[\exp(\alpha \sum_{t=1}^T\delta_t)\right] \leq \exp(T\alpha^2\sigma^2).
 \end{align}
 Now applying the Markov's inequality we have that for all $\alpha>0, M_{\delta}$:
 \begin{align}\label{eq:delta_bound}
 p(\sum_{t=1}^T\delta_t\geq M_\delta) &\leq \frac{\E\left[\exp(\alpha \sum_{t=1}^T\delta_t)\right]}{\exp(\alpha M_\delta) } \stackrel{\eqref{eq:all_sum}}{\leq} \frac{\exp(T\alpha^2\sigma^2)}{\exp(\alpha M_\delta)}.
 \end{align}
 Now for deriving bounds for our solution we have that for any $M>0$:
 
 \begin{align}\label{eq:tail}
    &P\left({\scriptsize F^{(k)}(\bar{\vc{\theta}}^T_1) + G(\bar{\vc{\theta}}^{T}_2)- F^{(k)}(\vc{\theta}^*)- G(\vc{\theta}^*) \geq \frac{(N^{1/p}d M_{D} + hD_{\mathcal{C}})^2}{2hT}\log(T+1) + \frac{\beta+h}{2T} D_{\mathcal{C}}^2 + \frac{M}{\sqrt{T}}}\right) \nonumber\\
    &\stackrel{\eqref{eq:proof}}{\leq} P\left(\frac{1}{T}\sum_{t=1}^T \delta_t\geq \frac{M}{\sqrt{T}}\right) \nonumber\\
    &= P\left(\sum_{t=1}^T \delta_t\geq {M\sigma\sqrt{T}}\right)\nonumber\\ &\stackrel{\eqref{eq:delta_bound}}{\leq} \exp\left(-\frac{M^2}{4}\right),
 \end{align}
 where for deriving the last inequality we  set $M_\delta=M\sigma\sqrt{T}$ and $\alpha = \frac{M}{2\sigma\sqrt{T}}$ in \eqref{eq:delta_bound}. Eq.~\eqref{eq:tail} is equivalent to \eqref{eq:oadm_opt_bound}, where we set
 \begin{subequations}\label{eq:k1k2k3}
 \begin{align}
 k_1 & \triangleq \max\left( \frac{(N^{1/p}d M_{D} + hD_{\mathcal{C}})^2\log(3)}{2h\log(2)},  \frac{\beta+h}{2} D_{\mathcal{C}}^2\right)\\
 k_2 & \triangleq M.
  \end{align}
 \end{subequations}

 \begin{comment}
 In order to make the bound in \eqref{eq:tail} more interpretable, for a given $M>0$ let us set $T_M\geq \left(\frac{(N^{1/p}d M_{D} + hD_{\mathcal{C}})^2}{2h} + \frac{\beta+h}{2} D_{\mathcal{C}}^2 \right)^2\frac{4}{M^2},$ (i.e., $T_M = \Omega(\frac{1}{M^2})$) for which we have that 
 \begin{align}
     \frac{M}{2\sqrt{T_M}} &\geq \left(\frac{(N^{1/p}d M_{D} + hD_{\mathcal{C}})^2}{2h} + \frac{\beta+h}{2} D_{\mathcal{C}}^2 \right) \frac{1}{\sqrt{T_M}} \nonumber \\
    &\geq \frac{(N^{1/p}d M_{D} + hD_{\mathcal{C}})^2}{2hT_M}\log(T_M+1) + \frac{\beta+h}{2T_M} D_{\mathcal{C}}^2.\label{eq:tm_ineq}
 \end{align}
 Then we have the following:

 \begin{align*}
     & P\left(F^{(k)}(\bar{\vc{\theta}}^{T_M}_1) + G(\bar{\vc{\theta}}^{T_M}_2)- F^{(k)}(\vc{\theta}^*)- G(\vc{\theta}^*) \geq \frac{M}{\sqrt{T_M}}\right) \\
     &\stackrel{\eqref{eq:tm_ineq}}{\leq} P\left({\tiny F^{(k)}(\bar{\vc{\theta}}^{T_M}_1) + G(\bar{\vc{\theta}}^{T_M}_2)- F^{(k)}(\vc{\theta}^*)- G(\vc{\theta}^*) \geq \frac{(N^{1/p}d M_{D} + hD_{\mathcal{C}})^2}{2hT_M}\log(T_M+1) + \frac{\beta+h}{2T_M} D_{\mathcal{C}}^2 + \frac{M}{2\sqrt{T_M}}}\right) \\
     &\stackrel{\eqref{eq:tail}}{\leq} \exp(- \frac{M^2}{16}).
 \end{align*}

\end{comment}
\end{proof}

\section{Proof of Lemma~\ref{lem:oadm}}\label{append:proof_oadm_lemma}
\begin{proof}
We show that Assumption 3 of \cite{wang2013online} is satisfied, thus the results follow from Theorem 6 of  \cite{wang2013online}. For case (a), we need to show that the subgradient of the functions $F^{(k)}(\vc{\theta}; \vc{x}_t)$ are bounded. For any subgradient $\vc{g}\in \partial \left(\|F(\vc{\theta}^{(k)}; \vc{x}_t) + \vc{D}_t^{(k)}(\vc{\theta} -\vc{\theta}^{(k)}) \|_p \right)$ and for all $\vc{\theta}\in \reals^d$, we have
\begin{align*}
  \vc{g}^\top (\vc{\theta} - \vc{\theta}_1) &\leq \|F(\vc{\theta}^{(k)};\vc{x}_t) + \vc{D}_t^{(k)}(\vc{\theta} -\vc{\theta}^{(k)}) \|_p -   \|F(\vc{\theta}^{(k)};\vc{x}_t) + \vc{D}_t^{(k)}(\vc{\theta}_1 -\vc{\theta}^{(k)}) \|_p \\
  &\leq \| \vc{D}_t^{(k)}(\vc{\theta} - \vc{\theta}_1) \|_p \\
  &= \left(\sum_{i=1}^N |\vc{D}_i^\top (\vc{\theta} - \vc{\theta}_1)|^p\right)^{1/p} \\
  &\stackrel{\mbox{\tiny{Cauchy–Schwarz Ineq.}}}{\leq} \left(\sum_{i=1}^N (\|\vc{D}_i\|_2 \|\vc{\theta} - \vc{\theta}_1\|_2)^p\right)^{1/p}\\
  &\stackrel{\mbox{\eqref{eq:bounded_D}}}{\leq} N^{1/p}\sqrt{d} M_{D} \|\vc{\theta} - \vc{\theta}_1\|= M_G\|\vc{\theta} - \vc{\theta}_1\|,
\end{align*}
where  $\vc{D}_i$ is the $i$-th row of $\vc{D}_t^{(k)}$.
Now given that the above holds for all $\vc{\theta}$, we can show that for every element $i\in [d]$, $\vc{g}_i$ is bounded by $M_G$; to see this set $\vc{\theta} = \vc{\theta}_1 + e_i$ and it follows that $\vc{g}_i\leq M_G.$ 
We  therefore conclude that  the subgradients $\partial F^{(k)}(\vc{\theta}, \zeta_t) = \{\vc{g} + h (\vc{\theta} - \vc{\theta}^{(k)})| \vc{g}\in \partial F_{\vc{\theta}^{(k)}}(\vc{\theta}_1, \zeta_t)\}, \vc{\theta}\in \mathcal{C}$ are bounded:
\begin{align}\label{eq:bound_subgrad}
 \|\vc{g}_F\|_2 = \|\vc{g} + h (\vc{\theta} - \vc{\theta}^{(k)})\|_2 &\leq \|\vc{g}\|_2 + h \|(\vc{\theta} - \vc{\theta}^{(k)})\|_2 \nonumber \\
 &\leq \sqrt{d} M_G + h D_{\mathcal{C}}~~\forall \vc{g}_F\in \partial F^{(k)}(\vc{\theta}; \vc{x}_t),   \end{align}
 %So far we showed that (a) in Ass.~\ref{ass:const} of \cite{wang2013online} holds. 
 where the last inequality is due to the fact that $D_{\mathcal{C}}=\max_{\theta_1, \theta_2\in \mathcal{C}} \|\theta_1 - \theta_2\|_2.$
 
 Case (b) is satisfied with the choice of $\ell_2$ norm squared for Bregman distance, i.e., the term $\gamma\|\vc{\theta}_1 - \vc{\theta}_1^t\|_2^2$. Since we assumed the constraint set $\mathcal{C}$ is convex, closed and bounded, we take the diameter of the set $\mathcal{C}$ to be $$D_{\mathcal{C}}=\max_{\theta_1, \theta_2\in \mathcal{C}} \|\theta_1 - \theta_2\|_2.$$
 Thus, case (c) is satisfied as a result of initialization ($\vc{\theta}_1^1=\vc{\theta}_2^1=\vc{u}^1=0$), and the fact that $\|\vc{\theta}_1^1 - \vc{\theta}^*\|_2\leq D_{\mathcal{C}}$ and $\|\vc{\theta}_2^1 - \vc{\theta}^*\|_2\leq D_{\mathcal{C}}$. Case (d) is directly included in the assumption of Theorem~\ref{trm:oadm}. Finally, for  Case (e) we have
 \begin{align*}
     &\left|F^{(k)}(\vc{\theta}_1^{t+1}; \vc{x}_{t}) +G(\vc{\theta}_2^{t+1}) - \left(F^{(k)}(\vc{\theta}^*; \vc{x}_t) +G(\vc{\theta}^*  )\right)\right| \\
     &\leq \left|F^{(k)}(\vc{\theta}_1^{t+1}, \zeta_{t}) - F^{(k)}(\vc{\theta}^*, \zeta_t)| + |G(\vc{\theta}_2^{t+1}) - G(\vc{\theta}^*  )\right|\\&\stackrel{\mbox{(a)}}{\leq} 
     (\sqrt{d} M_G + h D_{\mathcal{C}})\|\vc{\theta}_1^{t+1} - \vc{\theta}^*\|_2 + L_G\|\vc{\theta}_2^{t+1} - \vc{\theta}^*\|_2 + \\& |\chi_{\mathcal{C}}(\vc{\theta}_2^{t+1})-\chi_{\mathcal{C}}(\vc{\theta}^*)|\\&\stackrel{\mbox{(b)}}{\leq} 
     (\sqrt{d} M_G + h D_{\mathcal{C}} + L_G)D_{\mathcal{C}},
 \end{align*}
 where we derive {(a)} using the  Lipschitz continuity of $g(\cdot)$ along with Lemma~2.6 from \cite{shalev2011online}, which states that $F^{(k)}$ is Lipschitz continuous if and only if some $\ell_p$ norm of its subgradients is bounded (that we showed in \eqref{eq:bound_subgrad}). Moreover, in deriving {(b)} we use the fact that the constraint set $\mathcal{C}$ is convex, closed, and bounded, and that $\chi_{\mathcal{C}}(\vc{\theta}_2^{t+1})=\chi_{\mathcal{C}}(\vc{\theta}^*)=0$, as $\vc{\theta}_2^{t+1}, \vc{\theta}^*\in \mathcal{C}.$ 
 So far we have shown that all cases in %Ass.~\ref{ass:const} 
 {Assumption 3} of \cite{wang2013online} are satisfied. Also, both $F^{(k)}$ and $G$ are $h$ and $\beta$ strongly convex, respectively; the former is due to the quadratic term and the latter is explicitly stated in our assumptions. Therefore, we have shown that all assumptions in Theorem 6 \cite{wang2013online} are satisfied and the results in \eqref{eq:bounds_oadm} follow from the theorem.  
\end{proof}

\section{Batch of Independently Random Samples} \label{append:corbatch}
\begin{corollary}(Batch Setting)
Assume the assumptions of Theorem~\ref{trm:oadm} and let the sequence $\{\vc{\theta}_1^t, \vc{\theta}_2^tm \vc{u}^t\}, t\in[T]$ be generated by OADM algorithm, i.e.,  \eqref{eq:oadm}, where the step \eqref{eq:update_theta1} is replaced with the following step:
\begin{align}
    \label{eq:update_theta1_batch}
    \vc{\theta}_1^{t+1} :=& \argmin_{\vc{\theta_1}} \frac{1}{J} \sum_{j=1}^J F^{(k)}(\vc{\theta}_1; \vc{x}^j_t) + \frac{\rho}{2}\|\vc{\theta}_1 - \vc{\theta}_2^t +\vc{u}^t\|_2^2 + \frac{\gamma}{2} \|\vc{\theta}_1 - \vc{\theta}_1^t\|_2^2,
\end{align}
where at each iteration $t\in[T]$, the points $\vc{x}^j_t\in \{\vc{x}_1,\ldots, \vc{x}_n\}$ are i.i.d. samples drawn uniformly at random. Then all the results of Theorem~\ref{trm:oadm} in \eqref{eq:oadm_conv} hold.
\end{corollary}
\begin{proof}
We only need to adapt the definition in \eqref{eq:newFrandom} to this batch setting as follows:
\eqref{eq:newFrandom} as follows
\begin{equation} \label{eq:newFrandomJsamples_suppl}
    F^{(k)}(\vc{\theta};[\vc{x}^j]_{j\in [J]}) \triangleq \frac{1}{J} \sum_{j=1}^{J} F_{\vc{\theta}^{(k)}}(\vc{\theta}, \vc{x}^j) + \frac{h}{2} \|\vc{\theta} - \vc{\theta}^{(k)}\|_2^2.
\end{equation}
Then all the results follow similar to the proof presented in App.~\ref{append:proofoadm}.
\end{proof}

\section{Experimental Details}\label{app:algodet}

 \noindent\textbf{Algorithm Hyperparameters and Stopping Criteria.} For all algorithms we use a batch size of 8; we avoid using larger batch sizes as the computation time increases.
Additional details and stopping criteria used for each algorithm are as follows:
\begin{itemize}
    \item \mboa{}: We run Alg.~\ref{alg:MBO} for 20 iterations, however we stop earlier if we do no see an improvement in the objective. In the $k$-th iteration of Alg.~\ref{alg:MBO}, for solving the sub-problems (Line~\ref{alg_line:inner_prob}), we run iterations of  SADM \eqref{eq:oadm} for a maximum  of 200 rounds or  until the primal and dual residuals are less then $0.95 ^k \epsilon,$ where we set $\epsilon=0.1$ and $0.001$ for training auto-encoders and multi-target regression, respectively. 
    \item \mbob{}: Again, we run Alg.~\ref{alg:MBO} for 20 iterations and stop earlier if the objective does not improve. For solving the sub-problems (Line~\ref{alg_line:inner_prob}), we run the stochastic gradient descent (SGD) with the learning rate of $10^{-5}$ and $10^{-3}$, respectively, for training auto-encoders and multi-target regression. In all cases we  set the momentum parameter to 0.9 and run SGD for 500 iterations. 
    \item \sgd{}: We run the stochastic gradient descent (SGD) algorithm with momentum with the learning rate of $10^{-6}$ and  the momentum parameter set to 0.9. This corresponds to the algorithm by Mai and Johanssen \cite{mai2020convergence} applied to our setting.  We run \sgd{} for $10^3$ and $10^4$ iterations for  training auto-encoders and multi-target regression, respectively; we observe that the algorithm achieves its minimum within this number of iterations. At each iteration, \sgd{} evaluates the objective $F_{\texttt{OBJ}}$ and outputs the solution for the best observed objective. 
\end{itemize}

\noindent \textbf{Implementation.} We implement all the algorithms in Python 3.7 and using the  PyTorch backend.  We run  all algorithms  on CPU machines that have  Intel(R) Xeon(R) CPUs (E5-2680 v4) with 2.4GHz clock speed.% and 
%14 cores%, and features Hyper-threading. 

\noindent \textbf{Applications.} For both applications, we set the regularizer as $g(\vc{\theta})=\frac{0.001}{2} \|\vc{\theta}\|_2^2$ and do not consider a constraint set, i.e., $\mathcal{C}=\reals^d.$
\begin{itemize}
    \item \textbf{Training Autoencoders}. We use a neural network with two convolutional and two de-convolutional layers; the convolutional layers have 8 and 4 output channels, respectively, and $3\times 3$ kernel weights. The de-convolutional layers exactly mirror the convolutional layers. We do not apply zero padding or dilation for any of the layers and use a convolution step size of 1. We apply a soft-plus activation function after each layer. 
    
    \item \textbf{Multi-target Regression}. We use a network with two layers; the first layer is a 1-dimensional convolutional layer with the kernel size of 3 and no zero padding or   dilation and the step size of 1. The second layer is a fully-connected layer with  278 hidden units and output size of 16 (the target size). We again apply the soft-plus activation after each layer. 
\end{itemize}
Note that we choose soft-plus activation, i.e., a smooth version of the ReLu, to make sure that the functions $F(\vc{\theta};\vc{x})$ are smooth. 

\noindent \textbf{Datasets.} For each of the applications we use two datasets:
\begin{itemize}
    \item \textbf{Multi-Target Regression.} We use \texttt{SCM1d}, from the collection of regression data made available by  \cite{Spyromitros-Xioufis2016}. %\texttt{RF} concerns the prediction of future river flow at some specific locations; it has 576 predictors and 8 target values, with a total of 9,125 samples. 
    \texttt{SCM1d}\footnote{\href{http://users.auth.gr/espyromi/mtr/mtr-datasets.zip}{http://users.auth.gr/espyromi/mtr/mtr-datasets.zip}} is a supply chain management dataset comprising 9803 samples with 280 predictors and 16 targets. %For both datasets 
    We use 80 percent of data for training, i.e., solving \eqref{eq:problem}, and the rest of it as the test set. 
    \item \textbf{Auto-encoders.} \sloppy We use two well-known datasets \texttt{MNIST}\footnote{\href{https://pytorch.org/vision/stable/datasets.html\#mnist}{https://pytorch.org/vision/stable/datasets.htm\#mnist}} and \texttt{Fashion-MNIST}\footnote{\href{https://pytorch.org/vision/stable/datasets.html\#fashion-mnist}{https://pytorch.org/vision/stable/datasets.html\#fashion-mnist}}. They both have greyscale $28\times 28$ images with 60,000  training samples and 10,000 testing samples. \texttt{MNIST} contains handwritten digits with 10 classes and \texttt{Fashion-MNIST} contains images of clothing items from 10 classes. 
\end{itemize}

\fussy 

\section{Classification Task}\label{app:classification}
 For autoencoder tasks, we train a  logistic regression model, where the input features are the outputs of the first two convolutional layers used for encoding data. We set  the parameters of the encoder to be the solutions obtained by the respective autoencoder training algorithm. Using the encoder output for training  sets and the corresponding target labels, we  train the logistic regression, using the $\ell_2$ regularizer. We train the latter with  different regularizer coefficients (0.01, 0.1, and 1.0) and report the best observed accuracy on the test sets in Fig.~\ref{fig:cls_task}. Full experimental results for \texttt{MNIST} and  \texttt{Fashion-MNIST} are reported in Tables~\ref{tab:acc_mnist} and \ref{tab:acc_fmnist},  respectively.

%\section{Additional Experimental results}\label{app:experiments} 
%\noindent\textbf{Classification accuracy.} We present the accuracy on test sets for the logistic regression described in Sec.~\ref{sec:exp_cls} in Tables~\ref{tab:acc_mnist} and \ref{tab:acc_fmnist}, for \texttt{MNIST} and  \texttt{Fashion-MNIST}, respectively. 

\begin{table}[t]
\parbox{0.5\textwidth}{
\centering

\resizebox{0.49\textwidth}{!}{
\begin{tiny}
    \begin{tabular}{|c|c|c|c|c|}\hline
         $\outl$ & $p$ &\multicolumn{3}{|c|}{\texttt{Acc.}}\\\hline
                     &     & \mboa{} & \mbob{} &\sgd{}\\\hline
          %           \multicolumn{5}{|c|}{\texttt{MNIST}}\\\hline
          0.0 & $\ell_2^2$ & -&- &0.876\\
          0.0 & 2.0 &{0.926}& 0.894 &0.868\\
          0.0 & 1.5 & {0.930} &  0.882& 0.861\\
         0.0&  1.0 & \textbf{0.936} &0.904& 0.861\\\hhline{|=|=|=|=|=|}
        0.05 & $\ell_2^2$ & -  & - &0.928\\
         0.05 & 2.0 & {0.925}& 0.921 & 0.914\\
        0.05 & 1.5 &0.928 &{0.929}& 0.867\\
          0.05 & 1.0 & \textbf{0.935}  &0.930&0.863 \\\hhline{|=|=|=|=|=|}
            0.1& $\ell_2^2$ &  -& - &0.907\\
            0.1&  2.0 & {0.925}&0.924&0.916\\
            0.1&  1.5 & 0.927 &{0.930}& 0.877\\
           0.1&  1.0 & \textbf{0.935} &0.926& 0.873\\\hhline{|=|=|=|=|=|}
            0.2& $\ell_2^2$ & - &-  &0.897\\
            0.2&  2.0 & {0.928} &0.856&0.914\\
            0.2&  1.5 &  {0.928}&{0.928}& 0.892\\
           0.2&  1.0 & \textbf{0.931} &\textbf{0.931}& 0.894\\\hhline{|=|=|=|=|=|}
            0.3& $\ell_2^2$ & - & - &0.867 \\
            0.3&  2.0 & 0.923 &{0.925} &0.913\\
            0.3&  1.5 &0.927 &\textbf{0.929}&0.914\\
            0.3&  1.0 &{0.927} &0.922&0.918\\
            \hline
        \end{tabular}
    \end{tiny}
        }
        \caption{Accuracy of the classifiers for \texttt{MNIST} dataset.}
    \label{tab:acc_mnist}
}
   \parbox{0.5\textwidth}{
\centering
\resizebox{0.49\textwidth}{!}{
    \begin{tiny}
        \begin{tabular}{|c|c|c|c|c|}\hline
        
         $\outl$ & $p$ &\multicolumn{3}{|c|}{\texttt{Acc.}}\\\hline
                     &     & \mboa{} & \mbob{} &\sgd{}\\\hline   
           % \multicolumn{5}{|c|}{\texttt{FashionMNIST}}\\\hline
           0.0 & $\ell_2^2$ & -& -&{0.754}\\
            0.0 & 2.0 &{0.836} & 0.824 &0.787\\
             0.0 & 1.5 &  {0.848}&  0.843& 0.793\\
         0.0&  1.0 &\textbf{0.853}  &\textbf{0.853}& 0.794\\\hhline{|=|=|=|=|=|}
         0.05 & $\ell_2^2$ &-  & - &{0.770}\\
        0.05 & 2.0 & {0.835} & 0.832 & 0.806\\
        0.05 & 1.5 &0.841 &{0.847}& 0.803\\
         0.05 & 1.0 & \textbf{0.850} &0.847&0.814 \\\hhline{|=|=|=|=|=|}
        0.1& $\ell_2^2$ &-  &  -&{0.806}\\
        0.1&  2.0 &0.835 &{0.836}&0.812\\
        0.1&  1.5 & 0.837 &{0.846}& 0.811\\
        0.1&  1.0 &\textbf{0.854}  &0.848& 0.783\\\hhline{|=|=|=|=|=|}
        0.2& $\ell_2^2$ &-  &-  &{0.754}\\
        0.2&  2.0 & 0.826 &{0.832}&0.821\\
        0.2&  1.5 & 0.815 &{0.837}& 0.815\\
        0.2&  1.0 & \textbf{0.849} &0.832& 0.758\\\hhline{|=|=|=|=|=|}
        0.3& $\ell_2^2$ &-  & - &{0.757} \\
        0.3&  2.0 & 0.817 &{0.827} &0.823\\
        0.3&  1.5 & {0.838}&0.837 & 0.826\\
        0.3&  1.0 & \textbf{0.846}& 0.819& 0.783\\
            \hline
        \end{tabular}
    \end{tiny}
    }
    \caption{Accuracy of the classifiers for \texttt{Fashion-MNIST} dataset.}
    \label{tab:acc_fmnist}
   }
\vspace{-5mm}
\end{table}

\section{Robustness Analysis}\label{app:robust}
Here we report non-outliers loss and test loss for \texttt{MNIST} for $p=1.5$. Similar to results for $p=1,2$ in Fig.~\ref{fig:ftest_nout_scale}, we see that both \mboa{} variants and \sgd{} generally stay constant w.r.t $\outl$, for $p=1.5$ as well. Moreover, MBO variants again obtain lower loss values. We also see that in the high outlier regime, $\outl=0.3,$ \mboa{} performs poorly, similar to Figures~\ref{fig:fout_scale_1} and \ref{fig:ftest_scale_1}.

\begin{figure}
    \subfloat[Non-outliers Loss, $p=1.5$]{\includegraphics[width=0.5\textwidth]{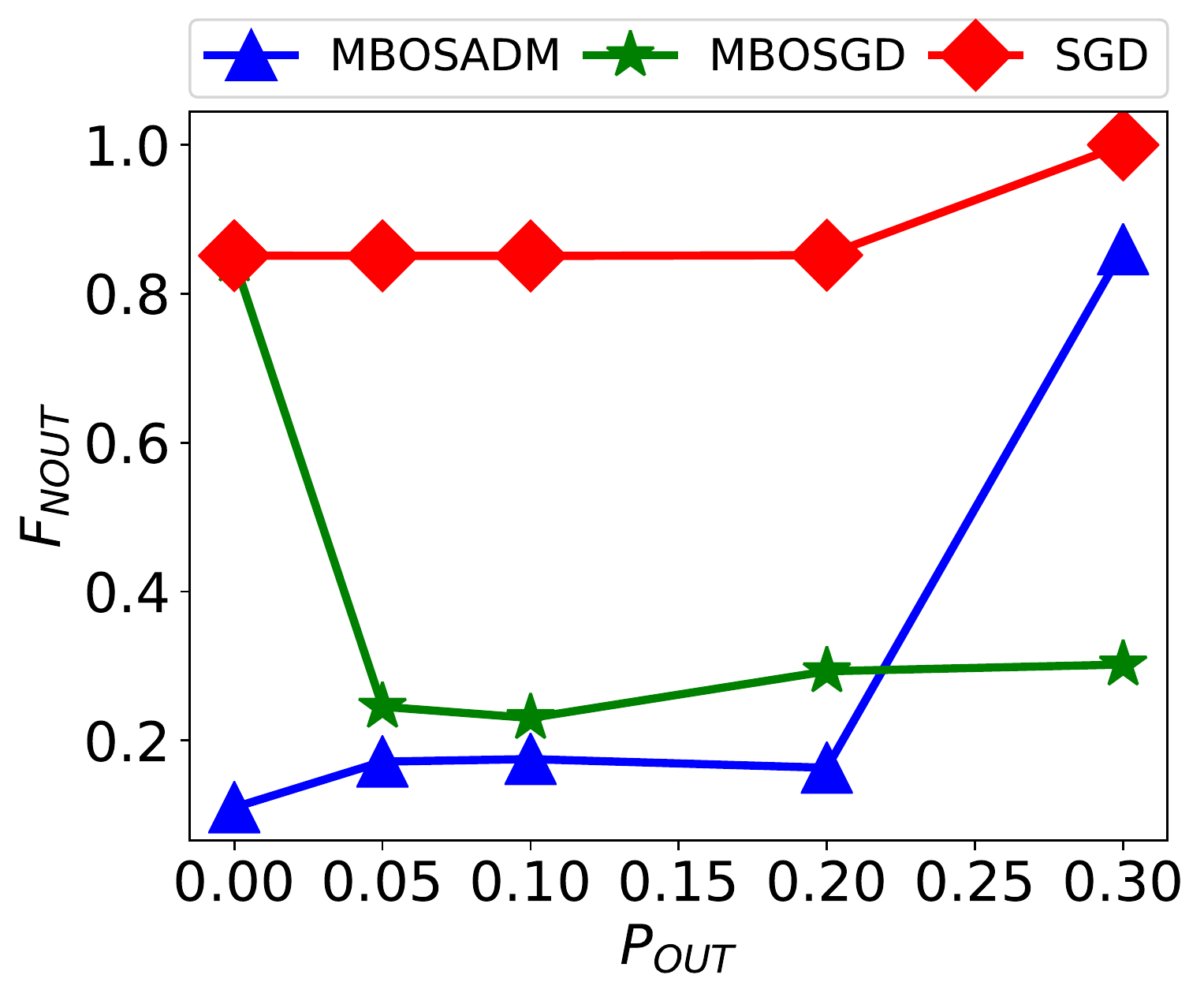}}
    \subfloat[Test Loss, $p=1.5$]{\includegraphics[width=0.5\textwidth]{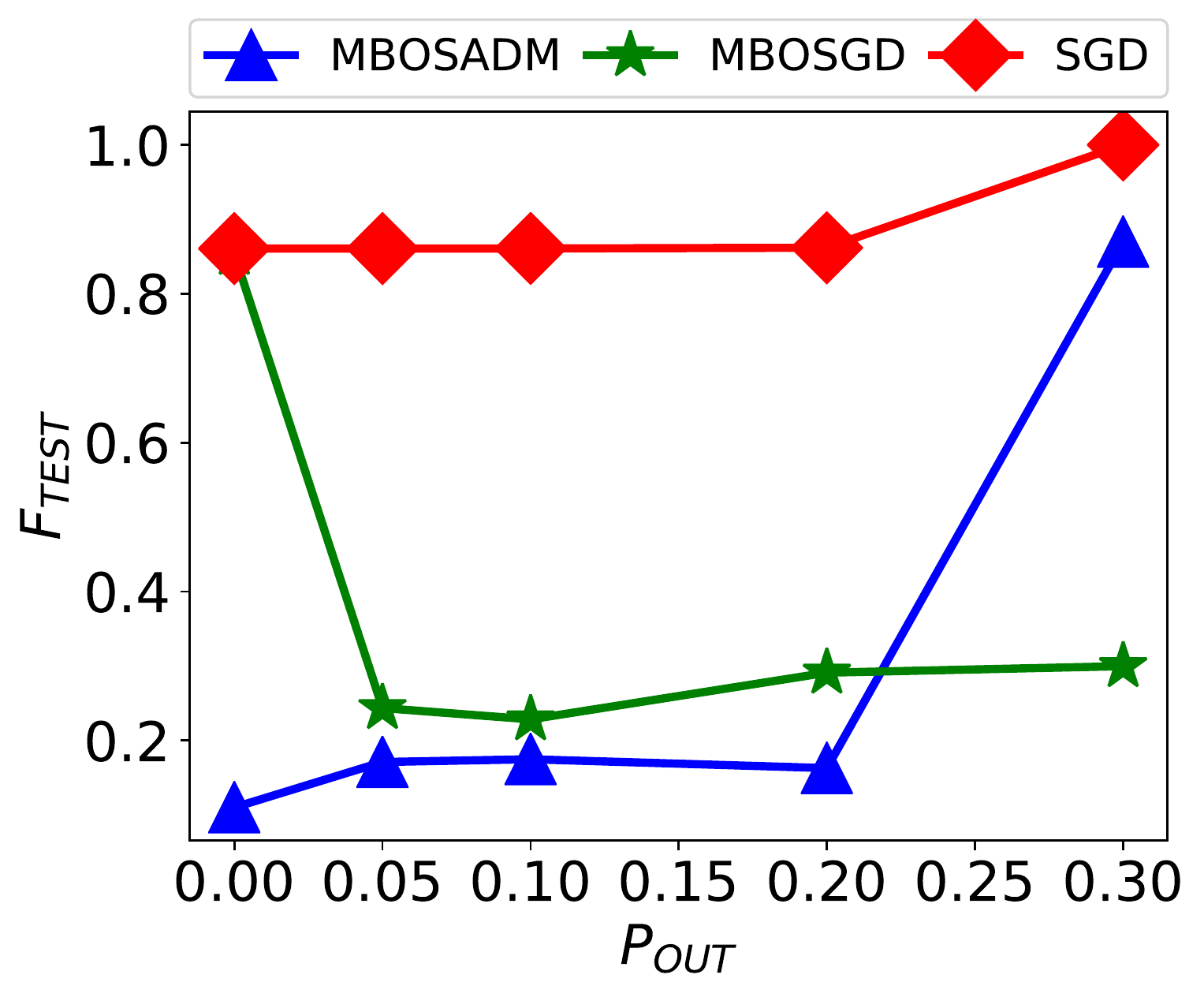}}
    \caption{Scalability  of the non-outliers loss $F_{\texttt{NOUT}}$ and the test loss $F_{\texttt{TEST}}$ for $p=1.5.$ Similar to Fig.~\ref{fig:ftest_nout_scale}, we see that both \mboa{} variants and \sgd{} generally stay constant w.r.t $\outl$ when for $p=1.5.$ We see that MBO variants obtain lower loss values. We also see that in the high outlier regime, $\outl=0.3,$ \mboa{} performs poorly, similar to Figures~\ref{fig:fout_scale_1} and \ref{fig:ftest_scale_1}.}\label{fig:ftest_nout_scale_app}
\end{figure}

\section{Proximal Operator for $\ell_p$ norms}\label{app:prox_op}
The proximal operator for $\ell_p$ norms ($p\geq1$), is the following:
\begin{align}\label{eq:POnorm}
\vspace{-2mm}
\textstyle 
\min_{\vc{u}\in\reals^{d}} \|\vc{u}\|_p+ \frac{\rho}{2} \|\vc{u}- \vc{w}\|_2^2,
\end{align}
for a given $\vc{w}\in \reals^{d}$.   Liu and Ye \cite{liu2010efficient}  define  a non-negative vector $\hat{\vc{w}}$ via
\begin{align}\hat{w}_i =\rho |w_i|~~\forall i\in[d].\label{eq:resc}\end{align}
They then consider the following simpler problem:
\begin{align}\textstyle\min_{\vc{u} \in \reals_+^d} \|\vc{u}\|_p +\frac{1}{2} \|\vc{u}- \hat{\vc{w}}\|_2^2. \label{eq:newprob} \end{align}
Note that this differs from Prob.~\eqref{eq:POnorm} in that  (a) $\rho=1$ and (b) vector $\vc{w}\in\reals^d$ replaced with non-negative vector $\vc{w}\in \reals^{d}_+$, and (c) optimization happens over $\vc{u}\in \reals_+^d$. Nevertheless,  Prob.~\eqref{eq:POnorm} is equivalent to Prob.~\eqref{eq:newprob} \cite{moharrer2020massively}.

\begin{algorithm}[!t]
%\algsetup{linenosize=\tiny}
    \caption{$p$-norm Prox. Operator}\label{alg:PO}
	\begin{footnotesize}

     \begin{algorithmic}[1]
     \State \textbf{Input:} $\vc{w}\in \reals^d$, $p\geq 1$, $\rho>0$, $\varepsilon>0$
     \State {Set $\hat{w}_i \leftarrow \rho |w_i|$ for $i=1,\ldots,d$.}
     \If{$\|\hat{\vc{w}}\|_q\leq 1$} \label{lin:norm1}
     \State {\textbf{Return}} $\vc{u}^*\leftarrow \vc{0}$
     \EndIf 
     \State Set ${\vc{u}}\leftarrow \vc{0}$,  $s_L \leftarrow 0$, and $s_U \leftarrow \|\hat{\vc{w}}\|_p$
     \For{$k=1,\ldots,\log_2\left\lceil \frac{1}{\varepsilon}\right\rceil $}
         \State Set $s \leftarrow (s_L + s_U)/2$
         \State Compute ${u}_i \leftarrow \hat{w}_i g\left(s \cdot (\hat{w}_i)^{\frac{2-p}{p-1}}\right)$ for all  $i\in [d]$;\label{enum:algo_PO_norm_step_A}
         \State Compute $\|{\vc{u}}\|_p$;\label{lin:norm2}
         \If{ ${\|{\vc{u}}\|_p} < s$}
         \State Set $s_U \leftarrow s$
         \Else 
         \State Set $s_L \leftarrow s$
         \EndIf
     
     \EndFor 
      \State Set $u_i^* \leftarrow \frac{\sign w_i}{\rho} {u}_i$ for $i=1,\ldots,d$.
     \State {\textbf{Return}}$\vc{u}^*$ %=[u_i]_{i\in [d]}$
     \end{algorithmic}
	\end{footnotesize}
\end{algorithm}

Liu and Ye  \cite{li2010l1} and Moharrer et al. \cite{moharrer2020massively} then solve  Prob.~\eqref{eq:newprob}. To do so, they  define first an auxiliary function as follows. 
Given $\alpha\in(0,\infty)$, they define the function $\alpha\mapsto g(\alpha)$,  as the unique solution of the following equation over $x \geq 0$: %for $a>0$
\begin{align} \label{eq:simple_equation_PO_norm}
\left({x}/{\alpha}\right)^{p-1} + x - 1=0,
\end{align}
They extend $g$ to $[0,\infty)$ by setting $g(0)\equiv 0$ for $\alpha=0$, by definition. 
Having defined $g$, given a  vector $\hat{\vc{w}}\in \reals_+^d$, they also define functions $g_i:\reals_+\rightarrow \reals_+, i\in[d]$ as: 
\begin{align}\label{eq:gi}g_i(s)=\hat{w}_i \cdot g\big(s\cdot (\hat{w}_i)^{\frac{2-p}{p-1}}\big),\end{align}
as well as function $h:\reals_+ \rightarrow \reals$ as:
\begin{align}\label{eq:funch}
h(s) = \textstyle\big( \sum_{i=1}^d g_i(s)^p \big)^{\frac{1}{p}} -s.
\end{align}

The bisection algorithm \cite{liu2010efficient,moharrer2020massively} for solving \eqref{eq:newprob}, summarized in Alg.~\ref{alg:PO}, proceeds as follows: given $\hat{\vc{w}}\in \reals_+^d$,  it  tests whether the condition $\|\hat{\vc{w}}\|_q\leq 1$ holds; if so, it returns $\vc{u}^*=\vc{0}$. Otherwise, 
Liu and Ye \cite{liu2010efficient} show that the solution is a function of $s^*,$ i.e., the root of the equation $h(s^*)=0.$ 
Alg.~\ref{alg:PO} finds a root %$\vc{u}^*$ by %identifying its norm $s^*=\|\vc{u}^*\|_p$ as the unique solution to $h(s)=0$
%finding 
$s^*$ via bisecting $[0,\|\hat{\vc{w}}\|_p]$. That is, at each iteration, it maintains an upper ($s_U$) and lower ($s_L$) bound on $s^*$, initialized at the above values. By construction, function $h$ alternates signs on each of the two bounds: i.e., $h(s_L)h(s_U)\leq 0;$ at each iteration, the method  (a) computes the average $s=0.5(s_L+s_U)$, between the two bounds, (b) find the sign of $h$ on this average, and then (c) update the bounds accordingly. 
For more information on the guarantees of the algorithm see Theorems~4.2 and 4.3 in \cite{moharrer2020massively}. 
%As signs alternate, by the intermediate value theorem, $s^*$ is guaranteed to be in $[s_L,s_U]$ at all times.

%\lipsum[2]
%\end{adjustwidth}
\end{document}